\documentclass[11pt]{article}
\pdfoutput=1
\usepackage{comment, fullpage}
\usepackage{amssymb}
\usepackage{graphicx}
\usepackage{url}
\usepackage{setspace}
\usepackage[pdftex,bookmarksnumbered,bookmarksopen,
colorlinks,citecolor=blue,linkcolor=blue]{hyperref}
\usepackage{framed}
\usepackage{xcolor}
\usepackage{soul}

\usepackage{times}
\usepackage{enumitem}
\usepackage{varwidth}
\usepackage{graphicx}
\usepackage{wrapfig}
\usepackage{enumerate}

\usepackage[utf8]{inputenc} 
\usepackage[T1]{fontenc}    
\usepackage{hyperref}       
\usepackage{url}            
\usepackage{booktabs}       
\usepackage{amsfonts}       
\usepackage{nicefrac}       
\usepackage{microtype}      

\usepackage[tight]{subfigure}
\usepackage{graphicx}
\usepackage{appendix}
\usepackage{amsmath,amsfonts,amsthm}
\usepackage{algorithm,algorithmic}
\usepackage[algo2e,ruled,vlined]{algorithm2e}
\usepackage{mdwlist}
\usepackage{xspace}
\usepackage{color}
\usepackage{mathrsfs}
\usepackage{amssymb}
\usepackage{wrapfig}

\usepackage{booktabs}
\usepackage{comment}
\usepackage{enumitem}
\newtheorem{theorem}{Theorem}
\newtheorem{lemma}[theorem]{Lemma}
\newtheorem{proposition}[theorem]{Proposition}

\renewcommand{\a}{\mathbf{a}}
\renewcommand{\b}{\mathbf{b}}
\renewcommand{\c}{\mathbf{c}}
\newcommand{\e}{\mathbf{e}}
\renewcommand{\u}{\mathbf{u}}
\renewcommand{\v}{\mathbf{v}}
\newcommand{\w}{\mathbf{w}}
\newcommand{\x}{\mathbf{x}}
\newcommand{\B}{\mathbf{B}}
\newcommand{\C}{\mathbf{C}}

\newcommand{\E}{\mathbf{E}}
\newcommand{\I}{\mathbf{I}}
\renewcommand{\L}{\mathbf{L}}
\newcommand{\M}{\mathbf{M}}
\newcommand{\N}{\mathcal{N}}
\renewcommand{\O}{\mathcal{O}}
\renewcommand{\P}{\mathcal{P}}
\newcommand{\R}{\mathbb{R}}
\renewcommand{\S}{\mathbf{S}}
\newcommand{\T}{\mathbf{T}}
\newcommand{\U}{\mathbf{U}}
\newcommand{\V}{\mathbf{V}}
\newcommand{\X}{\mathbf{X}}
\newcommand{\Y}{\mathbf{Y}}
\newcommand{\rank}{\mathbf{rank}}
\newcommand{\orthc}{\mathbf{orth}_c}
\newcommand{\orthr}{\mathbf{orth}_r}

\newcommand{\0}{\mathbf{0}}
\newcommand{\1}{\mathbf{1}}
\renewcommand{\comment}[1]{}

\title{Noise-Tolerant Life-Long Matrix Completion via Adaptive Sampling}
\usepackage{times}

\author{
Maria-Florina Balcan  \\ Machine Learning Department \\ Carnegie Mellon University \\ \small{ninamf@cs.cmu.edu} \and
Hongyang Zhang \\ Machine Learning Department \\ Carnegie Mellon University \\ \small{hongyanz@cs.cmu.edu}
}

\date{}

\begin{document}
\maketitle

\begin{abstract}
We study the problem of recovering an incomplete $m\times n$ matrix of rank $r$ with columns arriving online over time. This is known as the problem of \emph{life-long matrix completion}, and is widely applied to recommendation system, computer vision, system identification, etc. The challenge is to design provable algorithms tolerant to a large amount of noises, with small sample complexity. In this work, we give algorithms achieving strong guarantee under two realistic noise models. In bounded deterministic noise, an adversary can add any bounded yet unstructured noise to each column. For this problem, we present an algorithm that returns a matrix of a small error, with sample complexity almost as small as the best prior results in the noiseless case. For sparse random noise, where the corrupted columns are sparse and drawn randomly, we give an algorithm that \emph{exactly} recovers an $\mu_0$-incoherent matrix by probability at least $1-\delta$ with sample complexity as small as $\O\left(\mu_0rn\log (r/\delta)\right)$. This result advances the state-of-the-art work and matches the lower bound in a worst case. We also study the scenario where the hidden matrix lies on a mixture of subspaces and show that the sample complexity can be even smaller. Our proposed algorithms perform well experimentally in both synthetic and real-world datasets.
\end{abstract}

\section{Introduction}
Life-long learning is an emerging object of study in machine learning, statistics, and many other domains~\cite{balcan2014efficient,carlson-aaai}. In machine learning, study of such a framework has led to significant advances in learning systems that continually learn many tasks over time and improve their ability to learn as they do so, like humans~\cite{gopnik2001babies}. A natural approach to achieve this goal is to exploit information from previously-learned tasks under the belief that some commonalities exist across the tasks~\cite{balcan2014efficient,warmuth2008randomized}. The focus of this work is to apply this idea of life-long learning to the matrix completion problem. That is, given columns of a matrix that arrive online over time with missing entries, how to approximately/exactly recover the underlying matrix by exploiting the low-rank commonality across each column.

Our study is motivated by several promising applications where life-long matrix completion is applicable. In recommendation systems, the column of the hidden matrix consists of ratings by multiple users to a specific movie/news; The news or movies are updated online over time but usually only a few ratings are submitted by those users. In computer vision, inferring camera motion from a sequence of online arriving images with missing pixels has received significant attention in recent years, known as the structure-from-motion problem; Recovering those missing pixels from those partial measurements is an important preprocessing step. Other examples where our technique is applicable include system identification, multi-class learning, global positioning of sensors, etc.

Despite a large amount of applications of life-long matrix completion, many fundamental questions remain unresolved. One of the long-standing challenges is designing \emph{noise-tolerant, life-long} algorithms that can recover the unknown target matrix with small error. In the absence of noise, this problem is not easy because the overall structure of the low rankness is unavailable in each round.
This problem is even more challenging in the context of noise, where an adversary can add any bounded yet unstructured noise to those observations and the error propagates as the algorithm proceeds. This is known as bounded deterministic noise. Another type of noise model that receives great attention is sparse random noise, where the noise is sparse compared to the number of columns and is drawn i.i.d. from a non-degenerate distribution.

\noindent \textbf{Our Contributions:} This paper tackles the problem of \emph{noise-tolerant, life-long} matrix completion and advances the state-of-the-art results under the two realistic noise models.
\begin{itemize}
\item
Under bounded deterministic noise, we design and analyze an algorithm that is robust to noise, with only a small output error (See Figure \ref{figure: synthetic data}). The sample complexity is almost as small as the best prior results in the noiseless case, provided that the noise level is small.
\item
Under sparse random noise, we give sample complexity that guarantees an \emph{exact} recovery of the hidden matrix with high probability. The sample complexity advances the state-of-the-art results (See Figure \ref{figure: synthetic data}) and matches the lower bound in the worst case of this scenario.
\item
We extend our result of sparse random noise to the setting where the columns of the hidden matrix lie on a mixture of subspaces, and show that smaller sample complexity suffices to exactly recover the hidden matrix in this more benign setting.
\item
We also show that our proposed algorithms
perform well experimentally in both synthetic and real-world datasets.
\end{itemize}

\comment{
Compressed sensing is a norm in many data compression problem, a concrete example being one-bit camera, where only the signs of the pixels are captured and we are expecting to recover the original image from these highly compressed signals with overwhelming probabilities. As a two-dimensional compressed sensing problem, matrix completion has been widely applied to recommendation system, image repairing, data compression, etc., of which the task is to complete a low-rank matrix with only partially observed entries. Existing work~\cite{Candes2009exact,Candes2010power,recht2011simpler} showed that $\O(nr\log^2 n)$ passive sampling suffices to \emph{exactly recover} the underlying $n\times n$ matrix $\L$ with rank $r$ by probabilities at least $1-n^{-10}$ under mild incoherent conditions. This target can be achieved by solving a one-shot convex program, in which the nuclear norm is used as a convex surrogate of the rank function, and we do the nuclear norm minimization under the constraint that the sampled entries are consistent with the observed values. When there is bounded noise, i.e., the noise perturbation $\E$ such that $\|\E\|_F\le\epsilon$ and we can observe partial entries of $\M=\L+\E$, Cand\`{e}s et al.~\cite{candes2010matrix} further showed that $n^2p$ passive sampling suffices to approximately recover the $n\times n$ underlying matrix $\L$ with Frobenius error bounded by $\O(\sqrt{n/p}\epsilon)$.

In more recent work,  Krishnamurthy et al.~\cite{krishnamurthy2014power} sampled the matrix adaptively by column. The basic idea is to take full advantage of the low-dimensional commonality: Passing over all columns, if the current (partially observed) column can be represented by the basis matrix at hand, then completing that column by the linear combination of the base vectors; Otherwise, fully sampling the current column and adding it to the basis matrix, and then processing the next column. Krishnamurthy et al.~\cite{krishnamurthy2014power} showed that this procedure requires just $\O(nr\log^2 r)$ sampling to exactly recover the ground-truth in the adaptive sampling/active learning framework, if the observations are \emph{noise-free}.

In another direction, life-long learning~\cite{balcan2014efficient,NELL-aaai15,carlson-aaai} is a promising framework for machine learning. Highly related to transfer learning, multi-task learning and online learning, the life-long learning fully utilizes the knowledge we learn before to help learn the future related tasks, ``learning to learn" like human beings. In this paper, we propose plugging matrix completion into this framework. More specifically, we are assuming that we have $n$ learning (matrix/column completion) problems that arrive online over time. The problems are all over $m$ rows. For each upcoming column, we are able to measure only a few elements and the underlying matrix in hindsight has a rank of $r$. As shown in this paper, we are able to reduce the sample complexity to only $\O(rn\log r)$ even in this challenging online and moreover, noisy setting by an adaptive sampling strategy.

A concrete application of the above-mentioned settings is collaborative filtering. Famous examples include the Netflix challenge, where given $m$ users and $n_i$ items/movies, the task is to complete/predict the ratings for all user/item pair from partial observations. In our settings, we are assuming that the $m$ users are fixed. Every period of time, $n_i$ movies are updated and so we have multiple matrix completion problems arriving online over time. Given that the users typically have relatively fixed preference, the overall data should lie in a common low-dimensional subspace, which corresponds to our assumption nicely.
}

\section{Preliminaries}
Before proceeding, we define some notations and clarify problem setup in this section.

\noindent \textbf{Notations:}
We will use bold capital letter to represent matrix, bold lower-case letter to represent vector, and lower-case letter to represent scalar. Specifically, we denote by $\M\in\R^{m\times n}$ the noisy observation matrix in hindsight. We denote by $\L$ the underlying clean matrix, and by $\E$ the noise. We will frequently use $\M_{:t}\in\R^{m\times 1}$ to indicate the $t$-th column of matrix $\M$, and similarly $\M_{t:}\in\R^{1\times n}$ the $t$-th row. For any set of indices $\Omega$, $\M_{\Omega:}\in\R^{|\Omega|\times n}$ represents subsampling the rows of $\M$ at coordinates $\Omega$.
Without confusion, denote by $\U$ the column space spanned by the matrix $\L$. Denote by $\widetilde{\U}$ the noisy version of $\U$, i.e., the subspace corrupted by the noise, and by $\widehat{\U}$ our estimated subspace. The superscript $k$ of $\widetilde{\U}^k$ means that $\widetilde{\U}^k$ has $k$ columns in the current round. $\P_{\U}$ is frequently used to represent the orthogonal projection operator onto subspace $\U$. We use $\theta(\a,\b)$ to denote the angle between vectors $\a$ and $\b$. For a vector $\u$ and a subspace $\V$, define $\theta(\u,\V)=\min_{\v\in\V}\theta(\u,\v)$. We define the angle between two subspaces $\U$ and $\V$ as $\theta(\U,\V)=\max_{\u\in\U}\theta(\u,\V)$.
For norms, denote by $\|\v\|_2$ the vector $\ell_2$ norm of $\v$. For matrix, $\|\M\|_F^2=\sum_{ij}\M_{ij}^2$ and $\|\M\|_{\infty,2}=\max_i\|\M_{i:}\|_2$, i.e., the maximum vector $\ell_2$ norm across rows. The operator norm is induced by the matrix Frobenius norm, which is defined as $\|\P\|=\max_{\|\M\|_F\le 1}\|\P \M\|_F$. If $\P$ can be represented as a matrix, $\|\P\|$ also denotes the maximum singular value.

\subsection{Problem Setup}
\label{sec: problem setup}
In the setting of life-long matrix completion, we assume that each column of the underlying matrix $\L$ is normalized to have unit $\ell_2$ norm, and arrives online over time. We are not allowed to get access to the next column until we perform the completion for the current one. This is in sharp contrast to the offline setting where all columns come at one time and so we are able to immediately exploit the low-rank structure to do the completion. In hindsight, we assume the underlying matrix is of rank $r$. This assumption enables us to represent $\L$ as $\L=\U\S$, where $\U$ is the dictionary (a.k.a. basis matrix) of size $m\times r$ with each column representing a latent metafeature, and $\S$ is a matrix of size $r\times n$ containing the weights of linear combination for each column $\L_{:t}$. The overall subspace structure is captured by $\U$ and the finer grouping structure, e.g., the mixture of multiple subspaces, is captured by the sparsity of $\S$. Our goal is to approximately/exactly recover the subspace $\U$ and the matrix $\L$ from a small fraction of the entries, possibly corrupted by noise, although these entries can be selected sequentially in a feedback-driven way.

\textbf{Noise Models:} We study two types of realistic noise models, one of which is the deterministic noise. In this setting, we assume that the $\ell_2$ norm of noise on each column is bounded by $\epsilon_{noise}$. Beyond that, no other assumptions are made on the nature of noise. The challenge under this noise model is to design an \emph{online} algorithm limiting the possible error propagation during the completion procedure.
Another noise model we study is the sparse random noise, where we assume that the noise vectors are drawn i.i.d. from \emph{any} non-degenerate distribution. Additionally, we assume the noise is sparse, i.e., only a few columns of $\L$ are corrupted by noise. Our goal is to \emph{exactly} recover the underlying matrix $\L$ with sample complexity as small as possible.

\textbf{Incoherence:}
Apart from the sample budget and noise level, another quantity governing the difficulty of the completion problem is the coherence parameter on the row/column space. Intuitively, the completion should perform better when the information spreads evenly throughout the matrix. To quantify this term, for subspace $\U$ of dimension $r$ in $\R^m$, we define
\begin{equation}
\label{equ: incoherence}
\mu(\U)=\frac{m}{r}\max_{i\in[m]}\|\P_{\U}\e_i\|_2^2,
\end{equation}
where $\e_i$ is the $i$-th column of the identity matrix.
Indeed, without \eqref{equ: incoherence} there is an identifiability issue in the matrix completion problem~\cite{candes2010matrix,Candes2009exact,Zhang2015AAAI}. As an extreme example, let $\L$ be a matrix with only one non-zero entry. Such a matrix cannot be exactly recovered unless we see the non-zero element. As in \cite{krishnamurthy2014power}, to mitigate the issue, in this paper we assume incoherence $\mu_0=\mu(\U)$ on the column space of the underlying matrix. This is in contrast to the classical results of Cand\`{e}s et al.~\cite{candes2010matrix,Candes2009exact}, in which one requires incoherence $\mu_0=\max\{\mu(\U),\mu(\V)\}$ on both the column and the row subspaces.

\textbf{Sampling Model:} Instead of sampling the entries passively by uniform distribution, our sampling oracle allows adaptively measuring entries in each round. Specifically, for any arriving column we are allowed to have two types of sampling phases: we can either uniformly take the samples of the entries, as the passive sampling oracle, or choose to request all entries of the column in an adaptive manner. This is a natural extension of the classical passive sampling scheme with wide applications. For example, in
network tomography, a network operator is interested in inferring latencies between hosts while injecting few packets into the network. The operator is in control of the network, thus can adaptively sample the matrix of pair-wise latencies. In particular, the operator can request full columns of the matrix by measuring one host to all others.
In gene expression analysis, we are interested in recovering a matrix of expression levels for various genes across a number of conditions. The high-throughput microarrays provide expression levels of all genes of interest across operating conditions, corresponding to revealing entire columns of the matrix.

\section{Main Results}
In this section, we formalize our life-long matrix completion algorithm, develop our main theoretical contributions, and compare our results with the prior work.

\subsection{Bounded Deterministic Noise}
To proceed, our algorithm streams the columns of noisy $\M$ into memory and iteratively updates the estimate for the column space of $\L$. In particular, the algorithm maintains an estimate $\widehat{\U}$ of subspace $\U$, and when processing an arriving column $\M_{:t}$, requests only a few entries of $\M_{:t}$ and a few rows of $\widehat{\U}$ to estimate the distance between $\L_{:t}$ and $\U$. If the value of the estimator is greater than a given threshold $\eta_k$, the algorithm requests the remaining entries of $\M_{:t}$ and adds the new direction $\M_{:t}$ to the subspace estimate; Otherwise, finds a best approximation of $\M_{:t}$ by a linear combination of columns of $\widehat{\U}$. The pseudocode of the procedure is displayed in Algorithm \ref{algorithm: never ending learning}. We note that our algorithm is similar to the algorithm of \cite{krishnamurthy2014power} for the problem of offline matrix completion without noise.  However, our setting, with the presence of noise (which might conceivably propagate through the course of the algorithm), makes our analysis significantly more subtle.

\begin{algorithm}[ht]
\caption{Noise-Tolerant Life-Long Matrix Completion under Bounded Deterministic Noise}
\begin{algorithmic}
\label{algorithm: never ending learning}
\STATE {\bfseries Input:} Columns of matrices arriving over time.
\STATE {\bfseries Initialize:} Let the basis matrix $\widehat{\U}^0=\emptyset$. Randomly draw entries $\Omega\subset [m]$ of size $d$ uniformly with replacement.
\STATE {\bfseries 1:} \textbf{For} $t$ from $1$ to $n$, \textbf{do}
\STATE {\bfseries 2:}\ \ \ \ \ \ \ \ (a) If $\|\M_{\Omega t}-\P_{\widehat{\U}_{\Omega:}^k}\M_{\Omega t}\|_2>\eta_k$
\STATE {\bfseries 3:}\ \ \ \ \ \ \ \ \ \ \ \ \ i. Fully measure $\M_{:t}$ and add it to the basis matrix $\widehat{\U}^k$. Orthogonalize $\widehat{\U}^k$.
\STATE {\bfseries 4:}\ \ \ \ \ \ \ \ \ \ \ \ \ ii. Randomly draw entries $\Omega\subset [m]$ of size $d$ uniformly with replacement.
\STATE {\bfseries 5:}\ \ \ \ \ \ \ \ \ \ \ \ \ iii. $k:=k+1$.
\STATE {\bfseries 6:}\ \ \ \ \ \ \ \ (b) Otherwise $\widehat{\M}_{:t}:=\widehat{\U}^k\widehat{\U}_{\Omega:}^{k\dag}\M_{\Omega t}$.
\STATE {\bfseries 7:} \textbf{End For}
\STATE {\bfseries Output:} Estimated range space $\widehat{\U}^K$ and the underlying matrix $\widehat{\M}$ with column $\widehat{\M}_{:t}$.
\end{algorithmic}
\end{algorithm}

The key ingredient of the algorithm is to estimate the distance between the noiseless column $\L_{:t}$ and the clean subspace $\U^k$ with only a few measurements with noise. To estimate this quantity, we downsample both $\M_{:t}$ and $\widehat{\U}^k$ to $\M_{\Omega t}$ and $\widehat{\U}_{\Omega :}^k$, respectively. We then project $\M_{\Omega t}$ onto subspace $\widehat{\U}_{\Omega :}^k$ and use the projection residual $\|\M_{\Omega t}-\P_{\widehat{\U}_{\Omega:}^k}\M_{\Omega t}\|_2$ as our estimator. A subtle and critical aspect of the algorithm is the choice of the threshold $\eta_k$ for this estimator. In the noiseless setting, we can simply set $\eta_k=0$ if the sampling number $|\Omega|$ is large enough --- in the order of $\O(\mu_0r\log^2 r)$, because $\O(\mu_0r\log^2 r)$ noiseless measurements already contain enough information for testing whether a specific column lies in a given subspace~\cite{krishnamurthy2014power}. In the noisy setting, however, the challenge is that both $\M_{:t}$ and $\widehat{\U}^k$ are corrupted by noise, and the error propagates as the algorithm proceeds. Thus instead of setting the threshold as $0$ always, our theory suggests setting $\eta_k$ proportional to the noise level $\sqrt{\epsilon_{noise}}$. Indeed, the threshold $\eta_k$ balances the trade-off between the estimation error and the sample complexity: a) if $\eta_k$ is too large, most of the columns are represented by the noisy dictionary and therefore the error propagates too quickly; b) In contrast, if $\eta_k$ is too small, we observe too many columns in full and so the sample complexity increases. Our goal in this paper is to capture this trade-off, providing a global upper bound on the estimation error of the life-long arriving columns while keeping the sample complexity as small as possible.

\subsubsection{Recovery Guarantee}

Our analysis leads to the following guarantee on the performance of Algorithm \ref{algorithm: never ending learning}.

\begin{theorem}[Robust Recovery under Deterministic Noise]
\label{theorem: noise tolerance bound}
Let $r$ be the rank of the underlying matrix $\L$ with $\mu_0$-incoherent column space. Suppose that the $\ell_2$ norm of noise in each column is upper bounded by $\epsilon_{noise}$. Set the parameters $d\ge c(\mu_0r+mk\epsilon_{noise})\log^2 (2n/\delta))$ and
$\eta_k=C\sqrt{dk\epsilon_{noise}/m}$ for global constants $c$ and $C$. Then with probability at least $1-\delta$, Algorithm \ref{algorithm: never ending learning} outputs $\widehat{\U}^K$ with $K\le r$ and outputs $\widehat{\M}$ with $\ell_2$ error $\|\widehat{\M}_{:t}-\L_{:t}\|_2\le \O\left(\frac{m}{d}\sqrt{k\epsilon_{noise}}\right)$\footnote{By our proof, the constant factor is $9$.} uniformly for all $t$, where $k\le r$ is the number of base vectors when processing the $t$-th column.
\end{theorem}

\begin{proof}[Proof of Theorem \ref{theorem: noise tolerance bound}]
We firstly show that our estimated subspace in each round is accurate. The key ingredient of our proof is a result pertaining the angle between the underlying subspace and the noisy one.
Ideally, the column space spanned by the noisy dictionary cannot be too far to the underlying subspace if the noise level is small. This is true only if \emph{the angle between the newly added vector and the column space of the current dictionary is large}, as shown by the following lemma.
\begin{lemma}
\label{lemma: angle between U and tilde U}
Let $$\U^k=\mathbf{span}\{\u_1,\u_2,...,\u_k\} \mbox{\qquad and\qquad} \widetilde{\U}^k=\mathbf{span}\{\widetilde{\u}_1,\widetilde{\u}_2,...,\widetilde{\u}_k\}$$ be two subspaces such that $\theta(\u_i,\widetilde{\u}_i)\le\epsilon_{noise}$ for all $i\in[k]$. Let $$\gamma_k=\sqrt{20k\epsilon_{noise}} \mbox{\qquad and\qquad} \theta(\widetilde{\u}_i,\widetilde{\U}^{i-1})\ge \gamma_i$$ for $i=2,...,k$. Then
$$
\theta(\U^k,\widetilde{\U}^k)\le \gamma_k/2.
$$
\end{lemma}

\begin{proof}
The proof is basically by induction on $k$. Instead, we will prove a stronger result by showing that the conclusion holds on subspaces $\U^k=\mathbf{span}\{\mathbf{W},\u_1,\u_2,...,\u_k\}$ and $\widetilde{\U}^k=\mathbf{span}\{\mathbf{W},\widetilde{\u}_1,\widetilde{\u}_2,...,\widetilde{\u}_k\}$ for arbitrary fixed subspace $\mathbf{W}$. The base case $k=1$ follows immediately from Lemma \ref{lemma: subspaces angle diff one vect} .
\begin{lemma}[Lemma 2.~\cite{balcan2014efficient}]
\label{lemma: subspaces angle diff one vect}
Let $\mathbf{W}=\mathbf{span}\{\w_1,\w_2,...,\w_{k-1}\}$, $\U=\mathbf{span}\{\w_1,\w_2,...,\w_{k-1},\u\}$, and $\widetilde{\U}=\mathbf{span}\{\w_1,\w_2,...,\w_{k-1},\widetilde{\u}\}$ be subspaces spanned by vectors in $\R^m$. Then
\begin{equation*}
\theta\left(\U,\widetilde{\U}\right)\le \frac{\pi}{2}\frac{\theta(\widetilde{\u},\u)}{\theta(\widetilde{\u},\mathbf{W})}.
\end{equation*}
\end{lemma}

Now suppose the conclusion holds for any index $\le k-1$. Let $\U_0^k=\mathbf{span}\{\U^{k-1},\widetilde{\u}_k\}$. Then for index $k$, we have
\begin{equation*}
\begin{split}
\theta(\U^k,\widetilde{\U}^k)&\le \theta(\U^k,\U_0^k)+\theta(\U_0^k,\widetilde{\U}^k)\\
&\le \frac{\pi}{2}\frac{\theta(\widetilde{\u}_k,\u_k)}{\theta(\widetilde{\u}_k,\U^{k-1})}+10(k-1)\frac{\epsilon_{noise}}{\gamma_{k-1}}\ \ (\text{By Lemma \ref{lemma: subspaces angle diff one vect} and induction hypothesis})\\
&\le \frac{\pi}{2}\frac{\epsilon_{noise}}{\theta(\widetilde{\u}_k,\widetilde{\U}^{k-1})-\theta(\U^{k-1},\widetilde{\U}^{k-1})}+10(k-1)\frac{\epsilon_{noise}}{\gamma_{k-1}}\\
&\le \frac{\pi}{2}\frac{\epsilon_{noise}}{\gamma_k-10(k-1)\frac{\epsilon_{noise}}{\gamma_{k-1}}}+10(k-1)\frac{\epsilon_{noise}}{\gamma_{k-1}}\ \ (\text{By induction hypothesis})\\
&= \frac{\epsilon_{noise}}{\gamma_{k-1}}\left(\frac{\pi}{2}\frac{\gamma_{k-1}^2}{\gamma_k\gamma_{k-1}-10(k-1)\epsilon_{noise}}+10(k-1)\right)\\&\le \frac{\epsilon_{noise}}{\gamma_{k-1}}(\pi+10k-10)\\
&= \frac{\epsilon_{noise}}{\gamma_{k}}\frac{\gamma_k}{\gamma_{k-1}}(\pi+10k-10)\\&= \frac{\epsilon_{noise}}{\gamma_{k}}\sqrt{\frac{k}{k-1}}(\pi+10k-10)\\&\le \frac{\epsilon_{noise}}{\gamma_{k}}10k\ \ \ \ (k\ge 2).
\end{split}
\end{equation*}
\end{proof}

We then prove the correctness of our test in Step 2. Lemma \ref{lemma: angle between U and tilde U} guarantees that the underlying subspace $\U^k$ and our estimated one $\widetilde{\U}^k$ cannot be too distinct. So by algorithm, projecting any vector on the subspace spanned by $\widetilde{\U}^k$ does not make too many mistakes, i.e., $\theta(\M_{:t},\widetilde{\U}^k)\approx \theta(\M_{:t},\U^k)$. On the other hand, by standard concentration argument our test statistic $\|\M_{\Omega t}-\P_{\widetilde{\U}_{\Omega:}^k}\M_{\Omega t}\|_2$ is close to $\frac{d}{m}\|\M_{: t}-\P_{\widetilde{\U}^k}\M_{: t}\|_2$. Note that the latter term is determined by the angle of $\theta(\M_{:t},\widetilde{\U}^k)$. Therefore, our test statistic in Step 2 is indeed an effective measure of $\theta(\M_{:t},\widetilde{\U}^k)$, or $\theta(\L_{:t},\widetilde{\U}^k)$ since $\L_{:t}\approx\M_{:t}$, as proven by the following novel result.

\begin{lemma}
\label{lemma: residual estimate}
Let $\epsilon_k=2\gamma_k$, $\gamma_k=\sqrt{20k\epsilon_{noise}}$, and $k\le r$. Suppose that we observe a set of coordinates $\Omega\subset [m]$ of size $d$ uniformly at random with replacement, where $d\ge c_0(\mu_0r+mk\epsilon_{noise})\log^2 (2/\delta)$. If $\theta(\L_{:t},\widetilde{\U}^k)\le \epsilon_k$, then with probability at least $1-4\delta$, we have
$
\|\M_{\Omega t}-\P_{\widetilde{\U}_{\Omega:}^k}\M_{\Omega t}\|_2\le C\sqrt{dk\epsilon_{noise}/m}.
$
Inversely, if $\theta(\L_{:t},\widetilde{\U}^k)\ge c\epsilon_k$, then with probability at least $1-4\delta$, we have
$
\|\M_{\Omega t}-\P_{\widetilde{\U}_{\Omega:}^k}\M_{\Omega t}\|_2\ge C\sqrt{dk\epsilon_{noise}/m},
$
where $c_0$, $c$ and $C$ are absolute constants.
\end{lemma}
\begin{proof}
The first part of the theorem follows from the upper bound of Lemma \ref{lemma: concentration of measure}. Specifically, by plugging $d$ into the lower bound of Lemma \ref{lemma: concentration of measure}, we see that $\alpha<1/2$ and $\gamma<1/3$. Note that
\begin{equation*}
\begin{split}
\left\|\L_{:t}-\P_{\widetilde{\U}^k}\L_{:t}\right\|_2&=\left\|\L_{:t}\right\|_2\sin \theta\left(\L_{:t},\P_{\widetilde{\U}^k}\L_{:t}\right)\\&\le \theta\left(\L_{:t},\P_{\widetilde{\U}^k}\L_{:t}\right)\\&=\theta(\L_{:t},\widetilde{\U}^k)\\&\le \epsilon_k.
\end{split}
\end{equation*}
Therefore, by Lemma \ref{lemma: concentration of measure},
\begin{equation*}
\begin{split}
\left\|\M_{\Omega t}-\P_{\widetilde{\U}_{\Omega :}^k}\M_{\Omega t}\right\|_2&\le \O\left(\sqrt{\frac{d}{m}}\left\|\M_{:t}-\P_{\widetilde{\U}^k}\M_{:t}\right\|_2\right)\\
&\le \O\left(\sqrt{\frac{d}{m}}\left(\left\|\L_{:t}-\P_{\widetilde{\U}^k}\L_{:t}\right\|_2+\left\|\P_{\widetilde{\U}^k}(\L_{:t}-\M_{:t})\right\|_2+\left\|\M_{:t}-\L_{:t}\right\|_2\right)\right)\\
&\le \O\left(\sqrt{\frac{d}{m}}(\epsilon_k+2\epsilon_{noise})\right)\le C\sqrt{\frac{dk\epsilon_{noise}}{m}}.
\end{split}
\end{equation*}

We now proceed the second part of the theorem. To this end, we first explore the relation between the incoherence of the noisy basis $\widetilde{\U}^k$ and the clean one $\U^k$. Since we are able to control the error propagation in $\widetilde{\U}^k$, intuitively, the incoherence of $\widetilde{\U}^k$ and $\U^k$ is not distinct too much. In particular, for any $i\in [m]$,
\begin{equation*}
\begin{split}
\left\|\P_{\widetilde{\U}^k}\e_i\right\|_2&\le \left\|\P_{\U^k}\e_i\right\|_2+\left\|\P_{\U^k}\e_i-\P_{\widetilde{\U}^k}\e_i\right\|_2\\
&\le \left\|\P_{\U^k}\e_i\right\|_2+\left\|\P_{\U^k}-\P_{\widetilde{\U}^k}\right\|\left\|\e_i\right\|_2\\
&= \left\|\P_{\U^k}\e_i\right\|_2+\left\|\e_i\right\|_2\sin\theta\left(\U^k,\widetilde{\U}^k\right)\\
&\le \left\|\P_{\U^k}\e_i\right\|_2+\theta\left(\U^k,\widetilde{\U}^k\right)\\
&\le \left\|\P_{\U^k}\e_i\right\|_2+\frac{\gamma_k}{2}\\
&= \left\|\P_{\U^k}\e_i\right\|_2+\frac{1}{4}\epsilon_k.
\end{split}
\end{equation*}
Therefore,
$
\mu\left(\widetilde{\U}^k\right)=\frac{m}{k}\max_{i\in[m]}\left\|\P_{\widetilde{\U}^k}\e_i\right\|_2^2\le \frac{m}{k}\left(2\left\|\P_{\U^k}\e_i\right\|_2^2+\frac{1}{8}\epsilon_k^2\right)\le 2\mu(\U^k)+c''m\epsilon_{noise},
$
for global constant $c''$. Also, note that
\begin{equation*}
\left\|\P_{\widetilde{\U}^k}\M_{:t}-\L_{:t}\right\|_2\ge \sin\theta\left(\L_{:t},\P_{\widetilde{\U}^k}\M_{:t}\right)\|\L_{:t}\|_2\ge \frac{1}{2}\theta\left(\L_{:t},\P_{\widetilde{\U}^k}\M_{:t}\right)\ge \frac{1}{2}\theta\left(\L_{:t},\widetilde{\U}^k\right)\ge \frac{c\epsilon_k}{2}.
\end{equation*}
So we have
\begin{equation*}
\begin{split}
&\left\|\M_{\Omega t}-\P_{\widetilde{\U}_\Omega^k}\M_{\Omega t}\right\|_2\ge \sqrt{\frac{1}{m}\left(\frac{d}{2}-\frac{3k\mu(\widetilde{\U}^k)\beta}{2}\right)}\left\|\M_{:t}-\P_{\widetilde{\U}^k}\M_{:t}\right\|_2\\
&\ge \Omega\left(\sqrt{\frac{d}{m}-\frac{3k\mu(\U^k)}{m}\log^2 (1/\delta)-c_0k\epsilon_{noise}\log^2 (1/\delta)}\left\|\M_{:t}-\P_{\widetilde{\U}^k}\M_{:t}\right\|_2\right)\\
&\ge \Omega\left(\sqrt{\frac{d}{m}-\frac{3\mu_0r}{m}\log^2 (1/\delta)-c_0k\epsilon_{noise}\log^2 (1/\delta)}\left\|\M_{:t}-\P_{\widetilde{\U}^k}\M_{:t}\right\|_2\right)\ \ (\mbox{Since } \U^k\subseteq\U^r)\\
&\ge \Omega\left(\sqrt{\frac{d}{m}-c_0k\epsilon_{noise}\log^2 (1/\delta)}\left(\left\|\P_{\widetilde{\U}^k}\M_{:t}-\L_{:t}\right\|_2-\left\|\L_{:t}-\M_{:t}\right\|_2\right)\right)\ \ \left(\mbox{Since }d>3\mu_0r\log^2(1/\delta)\right)\\
&> \Omega\left(\sqrt{\frac{d}{m}-c_0k\epsilon_{noise}\log^2 (1/\delta)}\left(\frac{c\epsilon_k}{2}-\epsilon_{noise}\right)\right)\\
&> C\sqrt{\frac{dk\epsilon_{noise}}{m}}\ \ \left(\mbox{Since }d> c_0mk\epsilon_{noise}\log^2(1/\delta)\right).
\end{split}
\end{equation*}
\end{proof}

Finally, as both our dictionary and our statistic are accurate, the output error cannot be too large. In particular, we first show $K\le r$. Notice that every time we add a new direction to the basis matrix if and only if Condition (a) in Algorithm \ref{algorithm: never ending learning} holds true. In that case by Lemma \ref{lemma: residual estimate}, if setting $\eta_k=C\sqrt{dk\epsilon_{noise}/m}$, then with probability at least $1-4\delta$, we have that $\theta(\L_{:t},\widetilde{\U}^k)\ge 2\gamma_k$, which implies $\theta(\M_{:t},\widetilde{\U}^k)\ge\theta(\L_{:t},\widetilde{\U}^k)-\theta(\M_{:t},\L_{:t}) \ge \gamma_k$. So by Lemma \ref{lemma: angle between U and tilde U}, $\theta(\U^k,\widetilde{\U}^k)\le\gamma_k/2$. Thus $\theta(\L_{:t},\U^k)\ge \theta(\L_{:t},\widetilde{\U}^k)-\theta(\U^k,\widetilde{\U}^k)\ge 3\gamma_k/2$. Since $\rank(\L)=r$, we obtain that $K\le r$.

We now proceed to prove the upper bound on the $\ell_2$ error in Theorem \ref{theorem: noise tolerance bound}. We discuss Case (a) and (b) respectively. If Condition (a) in Algorithm \ref{algorithm: never ending learning} holds true, then according to the algorithm, we fully observe $\M_{:t}$ and use it as our estimate $\widehat{\M}_{:t}$. So $\left\|\widehat{\M}_{:t}-\L_{:t}\right\|_2\le \epsilon_{noise}\le\Theta(\frac{m}{d}\sqrt{k\epsilon_{noise}})$; On the other hand, if Case (b) in Algorithm \ref{algorithm: never ending learning} holds true, then we represent $\widehat{\M}_{:t}$ by the basis subspace $\widetilde{\U}^k$. So we have
\begin{equation*}
\begin{split}
&\ \ \ \ \ \left\|\widehat{\M}_{:t}-\L_{:t}\right\|_2=\left\|\widetilde{\U}^k\widetilde{\U}_{\Omega:}^{k\dag}\M_{\Omega t}-\L_{:t}\right\|_2\\
&\le \left\|\widetilde{\U}^k\widetilde{\U}^{k\dag}\L_{: t}-\L_{:t}\right\|_2+\left\|\widetilde{\U}^k\widetilde{\U}_{\Omega:}^{k\dag}\L_{\Omega t}-\widetilde{\U}^k\widetilde{\U}^{k\dag}\L_{: t}\right\|_2+\left\|\widetilde{\U}^k\widetilde{\U}_{\Omega:}^{k\dag}\L_{\Omega t}-\widetilde{\U}^k\widetilde{\U}_{\Omega:}^{k\dag}\M_{\Omega t}\right\|_2\\
& = \sin\theta(\L_{:t},\widetilde{\U}^k)+\left\|\widetilde{\U}^k\widetilde{\U}_{\Omega:}^{k\dag}\L_{\Omega t}-\widetilde{\U}^k\widetilde{\U}^{k\dag}\L_{: t}\right\|_2+\left\|\widetilde{\U}^k\widetilde{\U}_{\Omega :}^{k\dag}(\L_{\Omega t}-\M_{\Omega t})\right\|_2.
\end{split}
\end{equation*}
To bound the second term, let $\L_{:t}=\widetilde{\U}^k\v+\e$, where $\widetilde{\U}^k\v=\widetilde{\U}^k\widetilde{\U}^{k\dag}\L_{:t}$ and $\|\e\|_2\le\epsilon_k$ since $\|\e\|_2=\sin\theta(\L_{:t},\widetilde{\U}^k)\le \epsilon_k$. So
\begin{equation*}
\begin{split}
\widetilde{\U}^k\widetilde{\U}_{\Omega:}^{k\dag}\L_{\Omega t}-\widetilde{\U}^k\widetilde{\U}^{k\dag}\L_{: t}&=\widetilde{\U}^k(\widetilde{\U}_{\Omega:}^{kT}\widetilde{\U}_{\Omega:}^{k})^{-1}\widetilde{\U}_{\Omega:}^{kT}(\widetilde{\U}_{\Omega:}^{k}\v+\e_{\Omega})-\widetilde{\U}_{\Omega:}^{k}\v\\
&=\widetilde{\U}^k\widetilde{\U}_{\Omega:}^{k\dag}\e_{\Omega}.
\end{split}
\end{equation*}
Therefore,
\begin{equation*}
\begin{split}
\left\|\widehat{\M}_{:t}-\L_{:t}\right\|_2&\le \theta(\L_{:t},\widetilde{\U}^k)+\left\|\widetilde{\U}^k\widetilde{\U}_{\Omega:}^{k\dag}\L_{\Omega t}-\widetilde{\U}^k\widetilde{\U}^{k\dag}\L_{: t}\right\|_2+\left\|\widetilde{\U}^k\widetilde{\U}_{\Omega:}^{k\dag}(\L_{\Omega t}-\M_{\Omega t})\right\|_2\\
&\le \theta(\L_{:t},\widetilde{\U}^k)+\left\|\widetilde{\U}^k\widetilde{\U}_{\Omega:}^{k\dag}\right\|\left\|\e_{\Omega}\right\|_2+\left\|\widetilde{\U}^k\widetilde{\U}_{\Omega:}^{k\dag}\right\|\left\|\L_{\Omega t}-\M_{\Omega t}\right\|_2\\
&\le \epsilon_k+\Theta\left(\frac{m}{d}\epsilon_k\right)+\Theta\left(\frac{m}{d}\epsilon_{noise}\right)\\
&=\Theta\left(\frac{m}{d}\sqrt{k\epsilon_{noise}}\right),
\end{split}
\end{equation*}
where $\left\|\widetilde{\U}^k\widetilde{\U}_{\Omega :}^k\right\|\le\sigma_1(\widetilde{\U}^k)/\sigma_k(\widetilde{\U}_{\Omega:}^k)\le\Theta(m/d)$ once $d\ge \Omega(\mu(\widetilde{\U}^k)k\log (k/\delta))$, due to Lemma \ref{lemma: matrix chernoff bound}.
The final sample complexity follows from the union bound on the $n$ columns.
\end{proof}

Theorem \ref{theorem: noise tolerance bound} implies a result in the noiseless setting when $\epsilon_{noise}$ goes to zero. Indeed, with the sample size growing in the order of $\O(\mu_0nr\log^2 n)$, Algorithm \ref{algorithm: never ending learning} outputs a solution that is exact with probability at least $1-\frac{1}{n^{10}}$. To the best of our knowledge, this is the best sample complexity in the existing literature for noiseless matrix completion without additional side information~\cite{krishnamurthy2014power,recht2011simpler}. For the noisy setting, Algorithm \ref{algorithm: never ending learning} enjoys the same sample complexity $\O(\mu_0nr\log^2 n)$ as the noiseless case, if $\epsilon_{noise}\le\Theta(\mu_0 r/(mk))$. In addition, Algorithm \ref{algorithm: never ending learning} inherits the benefits of adaptive sampling scheme. The vast majority results in the passive sampling scenarios require both the row and column incoherence for exact/robust recovery~\cite{recht2011simpler}. In contrast, via adaptive sampling we can relax the incoherence assumption on the row space of the underlying matrix and are therefore more applicable.

We compare our result with several related lines of research in the prior work. While lots of online matrix completion algorithms have been proposed recently, they either lack of solid theoretical guarantee~\cite{kennedy2014online}, or require strong assumptions for the streaming data~\cite{krishnamurthy2014power,lois2015online,dhanjal2014online,krishnamurthy2013low}.  Specifically, Krishnamurthy et al.~\cite{krishnamurthy2013low} proposed an algorithm that requires column subset selection in the noisy case, which might be impractical in the online setting as we cannot measure columns that do not arrive. Focusing on a similar online matrix completion problem, Lois et al.~\cite{lois2015online} assumed that a) there is a good initial estimate for the column space; b) the column space changes slowly; c) the base vectors of the column space are dense; d) the support of the measurements changes by at least a certain amount. In contrast, our assumptions are much simpler and more realistic.

We mention another related line of research --- matched subspace detection. The goal of matched subspace detection is to decide whether an incomplete signal/vector lies within a given subspace~\cite{balzano2010high,balzano2010online}. It is highly related to the procedure of our algorithm in each round, where we aim at determining whether an arriving vector belongs to a given subspace based on partial and noisy observations. Prior work targeting on this problem formalizes the task as a hypothesis testing problem. So they assume a specific random distribution on the noise, e.g., Gaussian, and choose $\eta_k$ by fixing the probability of false alarm in the hypothesis testing~\cite{balzano2010high,scharf1994matched}.
Compared with this, our result does not have any assumption on the noise structure/distribution.

\subsection{Sparse Random Noise}

In this section, we discuss life-long matrix completion on a simpler noise model but with a stronger recovery guarantee.
We assume that noise is sparse, meaning that the total number of noisy columns is small compared to the total number of columns $n$.  The noisy columns may arrive at any time, and each noisy column is assumed to be drawn i.i.d. from a non-degenerate distribution. Our goal is to \emph{exactly} recover the underlying matrix and identify the noise with high probability.

We use an algorithm similar to Algorithm \ref{algorithm: never ending learning} to attack the problem, with $\eta_k=0$.
The challenge is that here we frequently add noise vectors to the dictionary and so we need to distinguish the noise from the clean column and remove them out of the dictionary at the end of the algorithm. To resolve the issue, we additionally record the support of the representation coefficients in each round when we represent the arriving vector by the linear combinations of the columns in the dictionary matrix. On one hand, the noise vectors in the dictionary fail to represent any column, because they are random. So if the representation coefficient corresponding to a column in the dictionary is $0$ always, it is convincing to identify the column as a noise. On the other hand, to avoid recognizing a true base vector as a noise, we make a mild assumption that the underlying column space is identifiable. Typically, that means for each direction in the underlying subspace, there are at least two clean data points having non-zero projection on that direction. We argue that the assumption is indispensable, since without it there is an identifiability issue between the clean data and the noise. As an extreme example, we cannot identify the black point in Figures \ref{figure: identifiability} as the clean data or as noise if we make no assumption on the underlying subspace. To mitigate the problem, we assume that for each $i\in [r]$ and a subspace $\U^r$ with orthonormal basis, there are at least two columns $\L_{:a_i}$ and $\L_{:b_i}$ of $\L$ such that $[\U^r]_{:i}^T\L_{:a_i}\not=0$ and $[\U^r]_{:i}^T\L_{:b_i}\not=0$. The detailed algorithm can be found in Algorithm \ref{algorithm: never ending learning sparse noise}.
\begin{algorithm}
\caption{Noise-Tolerant Life-Long Matrix Completion under Sparse Random Noise}
\begin{algorithmic}
\label{algorithm: never ending learning sparse noise}
\STATE {\bfseries Input:} Columns of matrices arriving over time.
\STATE {\bfseries Initialize:} Let the basis matrix $\widehat{\B}^0=\emptyset$, the counter $\C=\emptyset$. Randomly draw entries $\Omega\subset [m]$ of size $d$ uniformly without replacement.
\STATE {\bfseries 1:} \textbf{For} each column $t$ of $\M$, \textbf{do}
\STATE {\bfseries 2:}\ \ \ \ \ \ \ \ (a) If $\|\M_{\Omega t}-\P_{\widehat{\B}_{\Omega:}^k}\M_{\Omega t}\|_2>0$
\STATE {\bfseries 3:}\ \ \ \ \ \ \ \ \ \ \ \ \ i. Fully measure $\M_{:t}$ and add it to the basis matrix $\widehat{\B}^k$.
\STATE {\bfseries 4:}\ \ \ \ \ \ \ \ \ \ \ \ \ ii. $\C:=[\C,0]$.
\STATE {\bfseries 5:}\ \ \ \ \ \ \ \ \ \ \ \ \ iii. Randomly draw entries $\Omega\subset [m]$ of size $d$ uniformly without replacement.
\STATE {\bfseries 6:}\ \ \ \ \ \ \ \ \ \ \ \ \ iv. $k:=k+1$.
\STATE {\bfseries 7:}\ \ \ \ \ \ \ \ (b) Otherwise
\STATE {\bfseries 8:}\ \ \ \ \ \ \ \ \ \ \ \ \ i. $\C:=\C+\1_{supp(\widehat{\B}_{\Omega:}^{k\dag}\M_{\Omega t})}^T$.\ \ \ \ //Record supports of representation coefficient
\STATE {\bfseries 9:}\ \ \ \ \ \ \ \ \ \ \ \ \ ii. $\widehat{\M}_{:t}:=\widehat{\B}^k\widehat{\B}_{\Omega:}^{k\dag}\M_{\Omega t}$.
\STATE {\bfseries 10:}\ \ \ \ \ \ \ $t:= t+1$.
\STATE {\bfseries 11:} \textbf{End For}
\STATE {\bfseries Outlier Removal:} Remove columns corresponding to entry 0 in vector $\C$ from $\widehat{\B}^{s_0+r}=[\E^{s_0},\U^r]$.
\STATE {\bfseries Output:} Estimated range space, identified outlier vectors, and recovered underlying matrix $\widehat{\M}$ with column $\widehat{\M}_{:t}$.
\end{algorithmic}
\end{algorithm}

\begin{wrapfigure}{R}{4cm}
\subfigure[Identifiable Subspace]{
\includegraphics[width=4cm]{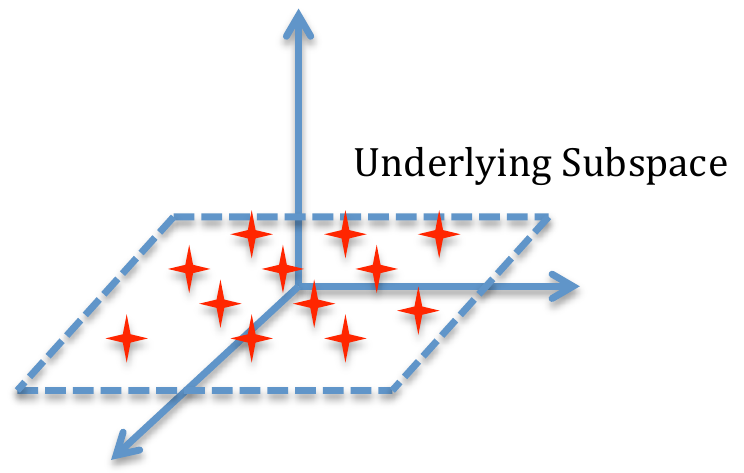}}
\subfigure[Unidentifiable Subspace]{
\includegraphics[width=4cm]{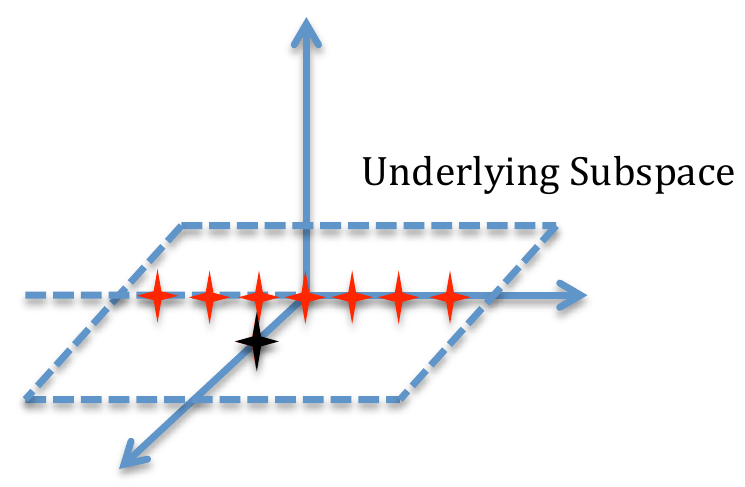}}
\caption{Identifiability.}
\label{figure: identifiability}
\end{wrapfigure}

\subsubsection{Upper Bound}
\label{section: upper bound}
We now provide upper and lower bound on the sample complexity of above algorithm for the exact recovery of underlying matrix. Our upper bound matches the lower bound up to a constant factor. We then analyze a more benign setting, namely, the data lie on a mixture of low-rank subspaces with dimensionality $\tau\ll r$.
Our analysis leads to the following guarantee on the performance of above algorithm.

\begin{theorem}[Exact Recovery under Random Noise]
\label{theorem: upper bound in noiseless case}
Let $r$ be the rank of the underlying matrix $\L$ with $\mu_0$-incoherent column space. Suppose that the noise $\E^{s_0}$ of size $m\times s_0$ are drawn from any non-degenerate distribution, and that the underlying subspace $\U^r$ is identifiable. Then our algorithm exactly recovers the underlying matrix $\L$, the column space $\U^r$, and the outlier $\E^{s_0}$ with probability at least $1-\delta$, provided that $d\ge c\mu_0r\log\left(r/\delta\right)$ and $s_0\le d-r-1$. The total sample complexity is thus $c\mu_0rn\log\left(r/\delta\right)$, where $c$ is a universal constant.
\end{theorem}

\begin{proof}[Proof of Theorem \ref{theorem: upper bound in noiseless case}]
We first prove a useful lemma which shows that the orthogonalization of a matrix does not change the rank of the matrix restricted on some rows/columns.
\begin{lemma}
\label{lemma: orth rank}
Let $\X=\U\mathbf{\Sigma}\V^T$ be the skinny SVD of $\X$, $\orthc(\X)=\U$, and $\orthr(\X)=\V^T$. Then for any set of coordinates $\Omega$ and any matrix $\X\in\R^{m\times n}$, we have
\begin{equation*}
\rank(\X_{\Omega:})=\rank([\orthc(\X)]_{\Omega:})\ \ \ \ \text{and}\ \ \ \ \rank(\X_{:\Omega})=\rank([\orthr(\X)]_{:\Omega}).
\end{equation*}
\end{lemma}
\begin{proof}
Let $\X=\U\mathbf{\Sigma}\V^T$ be the skinny SVD of matrix $\X$, where $\U=\orthc(\X)$ and $\V^T=\orthr(\X)$. On one hand,
\begin{equation*}
\X_{\Omega:}=\I_{\Omega:}\X=\I_{\Omega:}\U\mathbf{\Sigma}\V^T=[\orthc(\X)]_{\Omega:}\mathbf{\Sigma}\V^T.
\end{equation*}
So $\rank(\X_{\Omega:})\le\rank([\orthc(\X)]_{\Omega:})$. On the other hand, we have
\begin{equation*}
\X_{\Omega:}\V\mathbf{\Sigma}^{-1}=[\orthc(\X)]_{\Omega:}.
\end{equation*}
Thus $\rank([\orthc(\X)]_{\Omega:})\le\rank(\X_{\Omega:})$. So $\rank(\X_{\Omega:})=\rank([\orthc(\X)]_{\Omega:})$.

The second part of the argument can be proved similarly. Indeed, $\X_{:\Omega}=\U\mathbf{\Sigma}\V^T\I_{:\Omega}=\U\mathbf{\Sigma}[\orthr(\X)]_{:\Omega}$ and $\mathbf{\Sigma}^{-1}\U^T\X_{:\Omega}=[\orthr(\X)]_{:\Omega}$. So $\rank(\X_{:\Omega})=\rank([\orthr(\X)]_{:\Omega})$, as desired.
\end{proof}

We then investigate the effect of sampling on the rank of a matrix.
\begin{proposition}
\label{proposition: subsampling with same rank}
Let $\L\in\R^{m\times n}$ be any rank-$r$ matrix with skinny SVD $\U\mathbf{\Sigma}\V^T$. Denote by $\L_{:\Omega}$ the submatrix formed by subsampling the columns of $\L$ with i.i.d. Ber($d/n$). If $d\ge 8\mu(\V)r\log(r/\delta)$, then with probability at least $1-\delta$, we have $\rank(\L_{:\Omega})=r$. Similarly, denote by $\L_{\Omega:}$ the submatrix formed by subsampling the rows of $\L$ with i.i.d. Ber($d/m$). If $d\ge 8\mu(\U)r\log(r/\delta)$, then with probability at least $1-\delta$, we have $\rank(\L_{\Omega:})=r$.
\end{proposition}

\begin{proof}
We only prove the first part of the argument. For the second part, applying the first part to matrix $\L^T$ gets the result. Denote by $\T$ the matrix $\V^T=\orthr(\L)$ with orthonormal rows, and by $\X=\sum_{i=1}^n\delta_i\T_{:i}e_i^T\in\R^{r\times n}$ the sampling of columns from $\T$ with $\delta_i\sim\mbox{Ber}(d/n)$. Let $\X_i=\delta_i\T_{:i}e_i^T$. Define positive semi-definite matrix
\begin{equation*}
\Y=\X\X^T=\sum_{i=1}^n\delta_i\T_{:i}\T_{:i}^T.
\end{equation*}
Obviously, $\sigma_r^2(\X)=\lambda_r(\Y)$. To invoke the matrix Chernoff bound, we estimate the parameters $L$ and $\mu_r$ in Lemma \ref{lemma: matrix chernoff bound}. Specifically, note that
\begin{equation*}
\mathbb{E}\Y=\sum_{i=1}^n\mathbb{E}\delta_i\T_{:i}\T_{:i}^T=\frac{d}{n}\sum_{i=1}^n\T_{:i}\T_{:i}^T=\frac{d}{n}\T\T^T.
\end{equation*}
Therefore, $\mu_r=\lambda_r(\mathbb{E}\Y)=d\sigma_r^2(\T)/n>0$. Furthermore, we also have
\begin{equation*}
\lambda_{\max}(\X_i)=\|\delta_i\T_{:i}\|_2^2\le\|\T\|_{2,\infty}^2\triangleq L.
\end{equation*}
By the matrix Chernoff bound where we set $\epsilon=1/2$,
\begin{equation*}
\begin{split}
\Pr\left[\sigma_r(\X)> 0\right]&=\Pr\left[\lambda_r(\Y)> 0\right]\\
& \ge \Pr\left[\lambda_r(\Y)>\frac{1}{2}\mu_r\right]\\
& = \Pr\left[\lambda_r(\Y)>\frac{d}{2n}\sigma_r^2(\T)\right]\\
&\ge 1-r\exp\left(-\frac{d\sigma_r^2(\T)}{8n\|\T\|_{2,\infty}^2}\right)\\
&\triangleq 1-\delta.
\end{split}
\end{equation*}
So if
\begin{equation*}
d\ge \frac{8n\|\T\|_{2,\infty}^2}{\sigma_r^2(\T)}\log\frac{r}{\delta}=8n\|\T\|_{2,\infty}^2\log\left(\frac{r}{\delta}\right),
\end{equation*}
then $\Pr\left[\sigma_{k+1}(\X)=0\right]\le\delta$, where the last equality holds since $\sigma_r(\T)=\sigma_r(\V^T)=1$. Note that
\begin{equation*}
\|\T\|_{2,\infty}^2\le \max_{i\in[n]}\|\V^T \e_i\|_2^2\le \frac{r}{n}\mu(\V).
\end{equation*}
So if $d\ge 8\mu(\V)r\log (r/\delta)$ then with probability at least $1-\delta$, $\rank(\T_{:\Omega})=r$. Also, by Lemma \ref{lemma: orth rank}, $\rank(\T_{:\Omega})=\rank([\orthr(\L)]_{:\Omega})=\rank(\L_{:\Omega})$. Therefore, $\rank(\L_{:\Omega})=r$ with a high probability, as desired.
\end{proof}

\comment{
\begin{lemma}[Noiseless Case]
\label{lemma: noiseless rank argument}
Let $\U^k\in\R^{m\times k}$ be a $k$-dimensional subspace of $\U^r$, $\L_{:t}\in\U^r$ but $\L_{:t}\not\in \U^k$. Suppose we get access to a set of coordinates $\Omega\subset [m]$ of size $d$ uniformly at random with replacement. If $d\ge 8\mu_0r\log((k+1)/\delta)$ then with probability at least $1-\delta$, $\rank\left([\U_{\Omega:}^k,\L_{\Omega t}]\right)=k+1$. Inversely, if $\L_{:t}\in \U^k$ and $d\ge 8\mu_0r\log((k+1)/\delta)$, then $\rank\left([\U_{\Omega:}^k,\M_{\Omega t}]\right)=k$ with probability $1$, and $\U^k\U_{\Omega:}^{k\dag}\L_{\Omega t}=\L_{: t}$ with probability at least $1-\delta$.
\end{lemma}
\begin{proof}
The former part of the argument is an immediate result from Proposition \ref{proposition: subsampling with same rank}, where we have used the fact that $\mu_0r\ge \mu([\U^k,\L_{: t}])(k+1)$.

The latter part of the argument is from the fact that $\L_{:t}\in\U^k$ implies $\L_{\Omega t}\in \U_{\Omega:}^k$. Furthermore, from Proposition \ref{proposition: subsampling with same rank}, we have $\rank(\U_{\Omega:}^k)=\rank(\U^k)=k$ with probability $1-\delta$. So $(\U_{\Omega:}^{kT}\U_{\Omega:}^k)^{-1}$ exists. Let $\L_{:t}=\U^kv$ since $\L_{:t}\in\U^k$. Then $\L_{\Omega t}=\U_{\Omega:}^kv$. Therefore, we have that $\U^k\U_{\Omega:}^{k\dag}\L_{\Omega t}=\U^k\U_{\Omega:}^{k\dag}\L_{\Omega t}=\U^k(\U_{\Omega:}^{kT}\U_{\Omega:}^k)^{-1}\U_{\Omega:}^{kT}\L_{\Omega t}=\U^k(\U_{\Omega:}^{kT}\U_{\Omega:}^k)^{-1}\U_{\Omega:}^{kT}\U_{\Omega :}^kv=\U^kv=\L_{:t}$, as desired.
\end{proof}
}

\comment{
Lemma \ref{lemma: noiseless rank argument} is for the noise-free setting. To extend to the case of sparse corruption, note that Lemma \ref{lemma: facts on non-degenerate distribution} implies that replacing $\U^k$ with $[\E^s,\U^k]$ will not influence the result of Lemma \ref{lemma: noiseless rank argument} due to the randomness of $\E^s$. Indeed, the idea comes from the fact that random vectors from non-degenerate distribution do not concentrate on any specific subspace. So expanding $\U^k$ by $\E^s$ does not affect the representation of any fixed vector $\M_{:t}$: If $\M_{:t}$ can be represented by $\U^k$, so can $[\E^s,\U^k]$; Inversely, if $\M_{:t}$ cannot be represented by $\U^k$, it cannot be a linear combination of $[\E^s,\U^k]$ as well. Formally, our results are as follows:}

We now study the effectiveness of our representation step.
\begin{lemma}
\label{lemma: outlier rank argument}
Let $\U^k\in\R^{m\times k}$ be a $k$-dimensional subspace of $\U^r$. Suppose we get access to a set of coordinates $\Omega\subset [m]$ of size $d$ uniformly at random without replacement. Let $s\le d-r-1$ and $d\ge c\mu_0r\log(k/\delta)$ for a universal constant $c$.
\begin{itemize}
\item
If $\M_{:t}\in \U^r$ but $\M_{:t}\not\in \U^k$ then with probability at least $1-\delta$, $\rank\left([\E_{\Omega:}^s,\U_{\Omega:}^k,\M_{\Omega t}]\right)=s+k+1$.
\item
If $\M_{:t}\in \U^k$, then $\rank\left([\E_{\Omega:}^s,\U_{\Omega:}^k,\M_{\Omega t}]\right)=s+k$ with probability $1$, the representation coefficients of $\M_{:t}$ corresponding to $\E^s$ in the dictionary $[\E^s,\U^k]$ is $\0$ with probability $1$, and $[\E^s,\U^k][\E_{\Omega:}^{s},\U_{\Omega:}^{k}]^{\dag}\M_{\Omega t}=\M_{: t}$ with probability at least $1-\delta$.
\item
If $\M_{:t}\not\in \U^r$, i.e., $\M_{:t}$ is an outlier drawn from a non-degenerate distribution, then $\rank\left([\E_{\Omega:}^s,\U_{\Omega:}^k,\M_{\Omega t}]\right)=s+k+1$ with probability $1-\delta$.
\end{itemize}
\end{lemma}
\begin{proof}
\comment{
For the first part of the lemma, by Fact 3 of Lemma \ref{lemma: facts on non-degenerate distribution}, the rank of the matrix $[\E_{\Omega:}^s,\U_{\Omega:}^k,\M_{\Omega t}]$ is determined by that of $[\U_{\Omega:}^k,\M_{\Omega t}]$, namely, they have relation $\rank([\E^s,\U^k,\M_{:t}]_{\Omega:})=s+\rank([\U^k,\M_{:t}]_{\Omega:})$.}

For the first part of the lemma, note that $\rank([\U^k,\M_{:t}])=k+1$. So according to Proposition \ref{proposition: subsampling with same rank}, with probability $1-\delta$ we have that $\rank([\U^k,\M_{:t}]_{\Omega :})=k+1$ since $d\ge c\mu_0r\log((k+1)/\delta)\ge 8\mu([\U^k,\M_{:t}])k\log((k+1)/\delta)$ (Because $\M_{:t}\in\U^r$). Recall Facts 3 and 4 of Lemma \ref{lemma: facts on non-degenerate distribution} which imply that $\rank([\E^s,\U^k,\M_{:t}]_{\Omega :})=s+k+1$ when $s\le d-r-1$. This is what we desire.

\comment{
For the first part of the lemma, by Facts 2 and 4 of Lemma \ref{lemma: facts on non-degenerate distribution}, the $\E_{\Omega :}^s$ cannot linearly represent the given vector $\M_{\Omega t}$ due to the randomness. Thus the problem of whether $\M_{\Omega t}$ can be represented by a given dictionary $[\E_{\Omega:}^s,\U_{\Omega:}^k]$ is totally determined by whether it can be represented by $\U_{\Omega:}^k$, i.e., whether $\rank([\U_{\Omega:}^k,\M_{\Omega t}])=k+1$. By Proposition \ref{proposition: subsampling with same rank}, we have that once $d\ge 8\mu([\U^k,\M_{:t}]) (k+1)\log((k+1)/\delta)$, $\rank([\U_{\Omega:}^k,\M_{\Omega t}])=\rank([\U^k,\M_{:t}])=k+1$. Finally, the conclusion follows from the property that $\mu_0r\ge \mu([\U^k,\M_{: t}])(k+1)$.}

\comment{
The first part of the lemma is by Lemma \ref{lemma: noiseless rank argument} and the fact that $\rank([\E^s,\U^k,\M_{:t}]_{\Omega:})=s+\rank([\U^k,\M_{:t}]_{\Omega:})$ (Fact 3 of Lemma~\ref{lemma: facts on non-degenerate distribution}).}

For the middle part, the statement $\rank\left([\E_{\Omega:}^s,\U_{\Omega:}^k,\M_{\Omega t}]\right)=s+k$ comes from the assumption that $\M_{:t}\in\U^k$, which implies that $\M_{\Omega t}\in\U_{\Omega :}^k$ with probability $1$, and that $\rank\left([\E_{\Omega:}^s,\U_{\Omega:}^k]\right)=s+k$ when $s\le d-r-1$ (Facts 3 and 4 of Lemma \ref{lemma: facts on non-degenerate distribution}). Now suppose that the representation coefficients of $\M_{:t}$ corresponding to $\E^s$ in the dictionary $[\E^s,\U^k]$ is NOT $\0$ and $\M_{:t}\in \U^k$. Then $\M_{:t}-\U^k\c\in\mathbf{span}(\E^s)$, where $\c$ is the representation coefficients of $\M_{:t}$ corresponding to $\U^k$ in the dictionary $[\E^s,\U^k]$. Also, note that $\M_{:t}-\U^k\c\in\U^k$. So $\rank[\E^s,\M_{:t}-\U^k\c]=s$, which is contradictory with Fact 2 of Lemma \ref{lemma: facts on non-degenerate distribution}. So the coefficient w.r.t. $\E^s$ in the dictionary $[\E^s,\U^k]$ is $\0$, and we have that $[\E^s,\U^k][\E_{\Omega:}^{s},\U_{\Omega:}^{k}]^{\dag}\M_{\Omega t}=\U^k\U_{\Omega:}^{k\dag}\M_{\Omega t}=\U^k(\U_{\Omega:}^{kT}\U_{\Omega:}^k)^{-1}\U_{\Omega:}^{kT}\M_{\Omega t}=\U^k(\U_{\Omega:}^{kT}\U_{\Omega:}^k)^{-1}\U_{\Omega:}^{kT}\U_{\Omega :}^k\v=\U^k\v=\M_{: t}$, where $\v$ is the representation coefficient of $\M_{:t}$ w.r.t. $\U^k$. (The $(\U_{\Omega:}^{kT}\U_{\Omega:}^k)^{-1}$ exists because $\rank(\U_{\Omega:}^k)=k$ by Proposition \ref{proposition: subsampling with same rank})

As for the last part of the lemma, note that by Facts 2 and 4 of Lemma \ref{lemma: facts on non-degenerate distribution}, $\rank([\E^s,\M_{:t}]_{\Omega:})=s+1$. Then by Fact 3 of Lemma \ref{lemma: facts on non-degenerate distribution} and the fact that $\U_{\Omega:}^k$ has rank $k$ (Proposition \ref{proposition: subsampling with same rank}), we have $\rank\left([\E_{\Omega:}^s,\U_{\Omega:}^k,\M_{\Omega t}]\right)=s+k+1$ when $s\le d-r-1$, as desired.
\end{proof}

Now we are ready to prove Theorem \ref{theorem: upper bound in noiseless case}.
The proof of Theorem \ref{theorem: upper bound in noiseless case} is an immediate result of Lemma \ref{lemma: outlier rank argument} by using the union bound on the samplings of $\Omega$. Although Lemma \ref{lemma: outlier rank argument} states that, for a specific column $\M_{:t}$, the algorithm succeeds with probability at least $1-\delta$, the probability of success that uniformly holds for all columns is $1-(r+s_0)\delta$ rather than $1-n\delta$. This observation is from the proof of Lemma \ref{lemma: outlier rank argument}: $[\E^s,\U^k][\E_{\Omega:}^{s},\U_{\Omega:}^{k}]^{\dag}\M_{\Omega t}=\M_{: t}$ holds so long as $(\U_{\Omega:}^{kT}\U_{\Omega:}^k)^{-1}$ exists. Since in Algorithm \ref{algorithm: never ending learning sparse noise} we resample $\Omega$ if and only if we add new vectors into the basis matrix, which happens at most $r+s_0$ times, the conclusion follows from the union bound of the $r+s_0$ events. Thus, to achieve a global probability of $1-\delta$, the sample complexity for each upcoming column is $\Theta(\mu_0r\log(r+s_0/\delta))$. Since we also require that $s_0\le d-r-1$, the algorithm succeeds with probability $1-\delta$ once $d\ge \Theta(\mu_0r\log(d/\delta))$. Solving for $d$, we obtain that $d\gtrsim \mu_0r\log(\mu_0^2r^2/\delta^2)\asymp \mu_0r\log(r/\delta)$\footnote{We assume here that $\mu_0\le \mbox{poly}(r/\delta)$.}. The total sample complexity for Algorithm \ref{algorithm: never ending learning sparse noise} is thus $\Theta(\mu_0rn\log(r/\delta))$.

For the exact identifiability of the outliers, we have the following guarantee:

\begin{lemma}[Outlier Removal]
\label{lemma: outlier removal}
Let the underlying subspace $\U^r$ be identifiable, i.e., for each $i\in [r]$, there are at least two columns $\M_{:a_i}$ and $\M_{:b_i}$ of $\M$ such that $[\orthc(\U^r)]_{:i}^T\M_{:a_i}\not=0$ and $[\orthc(\U^r)]_{:i}^T\M_{:b_i}\not=0$. Then the entries of $\C$ in Algorithm \ref{algorithm: never ending learning sparse noise} corresponding to $\U^r$ cannot be $0$'s.
\end{lemma}

\begin{proof}
Without loss of generality, let $\U^r$ be orthonormal. Suppose that the lemma does not hold true. Then there must exist one column $\U_{:i}^r$ of $\U^r$, say e.g., $\e_i$, such that $\e_i^T\M_{:t}=0$ for all $t$ except when the index $t$ corresponds exactly to the $\U_{:i}^r$. This is contradictory with the condition that the subspace $\U^r$ is identifiable. The proof is completed.
\end{proof}

Thus the proof of Theorem \ref{theorem: upper bound in noiseless case} is completed.
\end{proof}

Theorem \ref{theorem: upper bound in noiseless case} implies an immediate result in the noise-free setting as $\epsilon_{noise}$ goes to zero. In particular, $\O\left(\mu_0nr\log (r/\delta)\right)$ measurements are sufficient so that our algorithm outputs a solution that is exact with probability at least $1-\delta$.
This sample complexity improves over existing results of $\O\left(\mu_0nr\log^2 (n/\delta)\right)$~\cite{recht2011simpler} and $\O\left(\mu_0nr^{3/2}\log (r/\delta)\right)$~\cite{krishnamurthy2013low}, and over $\O\left(\mu_0nr\log^2 (r/\delta)\right)$ of Theorem \ref{theorem: noise tolerance bound} when $\epsilon_{noise}=0$. Indeed, our sample complexity $\O\left(\mu_0nr\log (r/\delta)\right)$ matches the lower bound, as shown by Theorem \ref{theorem: lower bound in noiseless case} (See Table \ref{table: comparison of sample complexity} for comparisons of sample complexity). We notice another paper of Gittens~\cite{gittens2011spectral} which showed that Nsytr$\ddot{\mbox{o}}$m method recovers a \emph{positive-semidefinite} matrix of rank $r$ from uniformly sampling $\O(\mu_0r\log (r/\delta))$ columns. While this result matches our sample complexity, the assumptions of positive-semidefiniteness and of subsampling the columns are impractical in the online setting.
\begin{table}
\caption{Comparisons of our sample complexity with the best prior results in the noise-free setting.}
\label{table: comparison of sample complexity}
\centering
\begin{tabular}{c|c|cc}%
\hline
& Passive Sampling & \multicolumn{2}{c}{Adaptive Sampling}
\\ \hline
Upper Bound & $\O\left(\mu_0nr\log^2 (n/\delta)\right)$\cite{recht2011simpler} & $\O\left(\mu_0nr\log^2 (r/\delta)\right)$\cite{krishnamurthy2014power} & $\O\left(\mu_0nr\log (r/\delta)\right)$ (Ours)\\
Lower bound & $\O\left(\mu_0nr\log (n/\delta)\right)$\cite{Candes2010power} & \multicolumn{2}{c}{$\O\left(\mu_0nr\log (r/\delta)\right)$ (Ours)}\\ \hline
\end{tabular}
\end{table}
We compare Theorem \ref{theorem: upper bound in noiseless case} with prior methods on decomposing an incomplete matrix as the sum of a low-rank term and a column-sparse term. Probably one of the best known algorithms is Robust PCA via Outlier Pursuit~\cite{Xu:CSPCP,zhang2015relations,Zhang2015AAAI,zhang2015completing}. Outlier Pursuit converts this problem to a convex program:

\begin{equation}
\label{equ: outlier pursuit}
\min_{\L,\E} \|\L\|_*+\lambda\|\E\|_{2,1},\ \ \mbox{s.t.}\ \ \P_{\Omega}\M=\P_{\Omega}(\L+\E),
\end{equation}
where $\|\cdot\|_*$ captures the low-rankness of the underlying subspace and $\|\cdot\|_{2,1}$ captures the column-sparsity of the noise. Recent papers on Outlier Pursuit~\cite{zhang2015completing} prove that the solution to \eqref{equ: outlier pursuit} exactly recovers the underlying subspace, provided that $d\ge c_1\mu_0^2r^2\log^3 n$ and $s_0\le c_2d^4n/(\mu_0^5r^5m^3\log^6 n)$ for constants $c_1$ and $c_2$. Our result definitely outperforms the existing result in term of the sample complexity $d$, while our dependence of $s_0$ is not always better (although in some cases better) when $n$ is large. Note that while Outlier Pursuit loads all columns simultaneously and so can exploit the global low-rank structure, our algorithm is online and therefore cannot tolerate too much noise.

\subsubsection{Lower Bound}
We now establish a lower bound on the sample complexity. Our lower bound shows that in our adaptive sampling setting, one needs at least $\Omega\left(\mu_0rn\log\left(r/\delta\right)\right)$ many samples in order to uniquely identify a certain matrix \emph{in the worst case}. This lower bound matches our analysis of upper bound in Section \ref{section: upper bound}.

\comment{
To formalize the setting, we consider a sampling oracle in which one could take any strategy to learn the column space of the hidden matrix; After that, one can only sample the entries of arriving column uniformly at random. Such sampling oracle contains our adaptive sampling scheme as a special case, where our strategy to learn the column space is to request the whole column that does not belong to the current basis.
The following theorem provides lower bound on the sample complexity under our sampling scheme.}

\begin{theorem}[Lower Bound on Sample Complexity]
\label{theorem: lower bound in noiseless case}
Let $0<\delta<1/2$, and $\Omega\sim \mathrm{Uniform}(d)$ be the index of the row sampling $\subseteq [m]$. Suppose that $\U^r$ is $\mu_0$-incoherent. If the total sampling number
$dn< c\mu_0rn\log\left(r/\delta\right)$ for a constant $c$, then with probability at least $1-\delta$, there is an example of $\M$ such that under the sampling model of Section \ref{sec: problem setup} (i.e., when a column arrives the choices are either (a) randomly sample or (b) view the entire column), there exist infinitely many matrices $\L'$ of rank $r$ obeying $\mu_0$-incoherent condition on column space such that $\L'_{\Omega:}=\L_{\Omega:}$.
\end{theorem}

\begin{proof}
We prove the theorem by assuming that the underlying column space is known. Since we require additional samples to estimate the subspace, the proof under this assumption gives a lower bound. Let $\ell=\left\lfloor\frac{m}{\mu_0r}\right\rfloor$. Construct the underlying matrix $\L$ by
\begin{equation*}
\L=\sum_{k=1}^rb_k\u_k\u_k^T,
\end{equation*}
where the known $\u_k$ (Because the column space is known) is defined as
\begin{equation*}
\u_k=\sqrt{\frac{1}{\ell}}\sum_{i\in B_k}\e_i,\ \ \ \ B_k=\{(k-1)\ell+1,(k-1)\ell+2,...,k\ell\}.
\end{equation*}
So the matrix $\L$ is a block diagonal matrix formulated as Figure \ref{figure: construction of L}.
\begin{figure}[ht]
\centering
\includegraphics[width=0.6\textwidth]{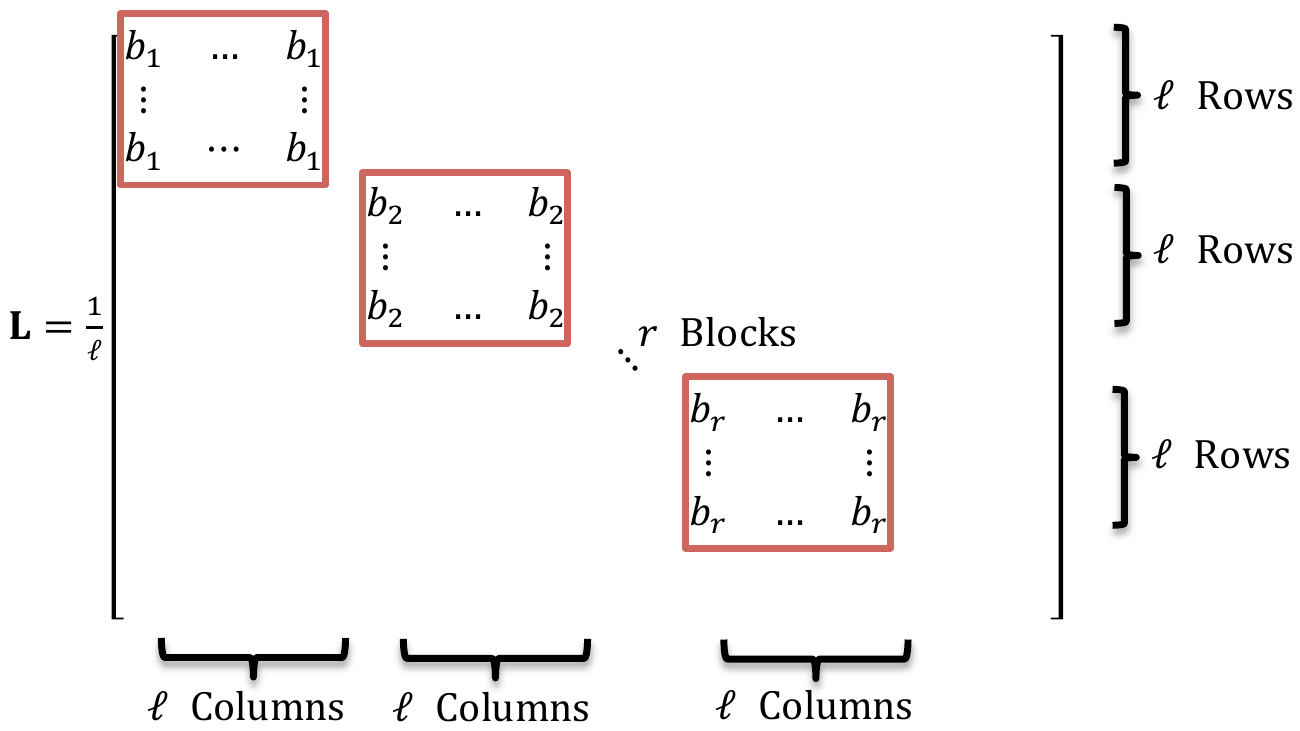}
\caption{Construction of underlying matrix $\L$.}
\label{figure: construction of L}
\end{figure}
Further, construct the noisy matrix $\M$ by
$
\M=[\L,\E].
$
The matrix $\E\in\R^{m\times s_0}$ corresponds to the outliers, and the matrix $\L$ corresponds to the underlying matrix.

Notice that the information of $b_k$'s is only implied in the corresponding block of $\L$.
So overall, the lower bound is given by solving from the inequality
\begin{equation*}
\Pr\{\mbox{For all blocks, there must be at least one row being sampled}\}\ge 1-\delta.
\end{equation*}
We highlight that the $b_k$'s can be chosen arbitrarily in that they do not change the coherence of the column space of $\L$. Also, it is easy to check that the column space of $\L$ is $\mu_0$-incoherent. By construction, the underlying matrix $\L$ is block-diagonal with $r$ blocks, each of which is of size $\ell\times \ell$. According to our sampling scheme, we always sample the same positions of the arriving column after the column space is known to us. This corresponds to sample the row of the matrix in hindsight. To recover $\L$, we argue that each block should have at least one row fully observed; Otherwise, there is no information to recover $b_k$'s. Let $A$ be the event that for a fixed block, none of its rows is observed. The probability $\pi_0$ of this event $A$ is therefore $\pi_0=(1-p)^\ell$, where $p$ is the Bernoulli sampling parameter. Thus by independence, the probability of the event that there is at least one row being sampled holds true \emph{for all diagonal blocks} is $(1-\pi_0)^{r}$, which is $\ge 1-\delta$ as we have argued. So
\begin{equation*}
-r\pi_0\ge r\log(1-\pi_0)\ge \log(1-\delta),
\end{equation*}
where the first inequality is due to the fact that $-x\ge \log(1-x)$ for any $x<1$. Since we have assumed $\delta<1/2$, which implies that $\log(1-\delta)\ge -2\delta$, thus $\pi_0\le 2\delta/r$. Note that $\pi_0=(1-p)^\ell$, and so
\begin{equation*}
-\log(1-p)\ge \frac{1}{\ell}\log\left(\frac{r}{2\delta}\right)\ge\frac{\mu_0 r}{m}\log\left(\frac{r}{2\delta}\right).
\end{equation*}
This is equivalent to
\begin{equation*}
mp\ge m\left(1-\exp\left(-\frac{\mu_0 r}{m}\log \frac{r}{2\delta}\right)\right).
\end{equation*}
Note that $1-e^{-x}\ge x-x^2/2$ whenever $x\ge 0$, we have
\begin{equation*}
mp\ge (1-\epsilon/2)\mu_0r\log\left(\frac{r}{2\delta}\right),
\end{equation*}
where $\epsilon=\mu_0r\log(r/2\delta)<1$. Finally, by the equivalence between the uniform and Bernoulli sampling models (i.e., $d\approx mp$, Lemma \ref{lemma: number of sampling by Bernoulli}), the proof is completed.
\end{proof}

We mention several lower bounds on the sample complexity for passive matrix completion. The first is the paper of Cand\`{e}s and Tao~\cite{Candes2010power}, that gives a lower bound of $\Omega(\mu_0nr\log(n/\delta))$ if the matrix has both incoherent rows and columns. Taking a weaker assumption, Krishnamurthy and Singh~\cite{krishnamurthy2013low,krishnamurthy2014power} showed that if the row space is \emph{coherent}, any passive sampling scheme followed by any recovery algorithm must have $\Omega(mn)$ measurements. In contrast, Theorem \ref{theorem: lower bound in noiseless case} demonstrates that in the absence of row-space incoherence, exact recovery of the matrix is possible with only $\Omega(\mu_0nr\log(r/\delta))$ samples, if the sampling scheme is adaptive.

\subsubsection{Extension to Mixture of Subspaces}
\begin{wrapfigure}{R}{5cm}
\subfigure[Single Subspace]{
\includegraphics[width=5cm]{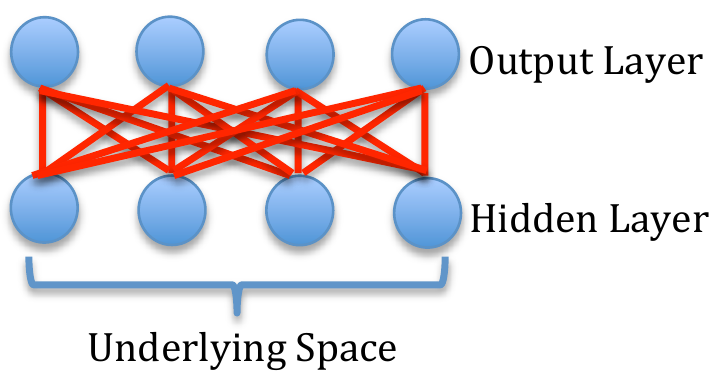}
}
\subfigure[Mixture of Subspaces]{
\includegraphics[width=5cm]{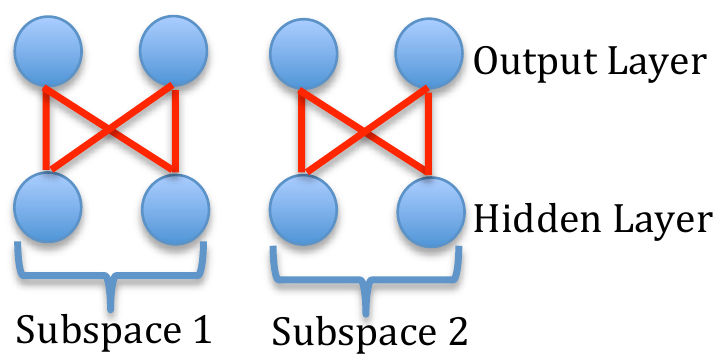}}
\caption{Subspace structure.}
\label{figure: subspace structure}
\end{wrapfigure}
Theorem \ref{theorem: lower bound in noiseless case} gives a lower bound on sample complexity in the \emph{worst} case. In this section, we explore the possibility of further reducing the sample complexity with more complex common structure. We assume that the underlying subspace is a mixture of $h$ independent subspaces\footnote{$h$ linear subspaces are independent if the dimensionality of their sum is equal to the sum of their dimensions.}~\cite{lerman2014l_p}, each of which is of dimension at most $\tau\ll r$. Such an assumption naturally models settings in which there are really $h$ different categories of movies/news while they share a certain commonality across categories.
We can view this setting as a network with two layers: The first layer captures the overall subspace with $r$ metafeatures; The second layer is an output layer, consisting of metafeatures each of which is a linear combination of only $\tau$ metafeatures in the first layer. See Figures \ref{figure: subspace structure} for visualization. Our argument shows that the sparse connections between the two layers significantly improve the sample complexity.

Algorithmically, given a new column, we uniformly sample $\tilde \O(\tau\log r)$ entries as our observations. We try to represent those elements by a \emph{sparse} linear combination of only $\tau$ columns in the basis matrix, whose rows are truncated to those sampled indices; If we fail, we measure the column in full, add that column into the dictionary, and repeat the procedure for the next arriving column. The detailed algorithm can be found in Algorithm \ref{algorithm: never ending learning sparse noise mixture of subspaces}.
\begin{algorithm}[ht]
\caption{Noise-Tolerant Life-Long Matrix Completion under Random Noise for Mixture of Subspaces}
\begin{algorithmic}
\label{algorithm: never ending learning sparse noise mixture of subspaces}
\STATE {\bfseries Input:} Columns of matrices arriving over time.
\STATE {\bfseries Initialize:} Let the basis matrix $\widehat{\B}^0=\emptyset$, the counter $\C=\emptyset$. Randomly draw entries $\Omega\subset [m]$ of size $d$ uniformly without replacement.
\STATE {\bfseries 1:} \textbf{For} each column $t$ of $\M$, \textbf{do}
\STATE {\bfseries 2:}\ \ \ \ \ \ \ \ (a) If there does not exist a $\tau$-sparse linear combination of columns of $\widehat{\B}_{\Omega :}^k$ that represents $\M_{\Omega t}$ exactly
\STATE {\bfseries 3:}\ \ \ \ \ \ \ \ \ \ \ \ \ i. Fully measure $\M_{:t}$ and add it to the basis matrix $\widehat{\B}^k$.
\STATE {\bfseries 4:}\ \ \ \ \ \ \ \ \ \ \ \ \ ii. $\C:=[\C,0]$.
\STATE {\bfseries 5:}\ \ \ \ \ \ \ \ \ \ \ \ \ iii. Randomly draw entries $\Omega\subset [m]$ of size $d$ uniformly without replacement.
\STATE {\bfseries 6:}\ \ \ \ \ \ \ \ \ \ \ \ \ iv. $k:=k+1$.
\STATE {\bfseries 7:}\ \ \ \ \ \ \ \ (b) Otherwise
\STATE {\bfseries 8:}\ \ \ \ \ \ \ \ \ \ \ \ \ i. $\C:=\C+\1_{supp(\widehat{\B}_{\Omega:}^{k\dag}\M_{\Omega t})}^T$.\ \ \ \ //Record supports of representation coefficient
\STATE {\bfseries 9:}\ \ \ \ \ \ \ \ \ \ \ \ \ ii. $\widehat{\M}_{:t}:=\widehat{\B}^k\widehat{\B}_{\Omega:}^{k\dag}\M_{\Omega t}$.
\STATE {\bfseries 10:}\ \ \ \ \ \ \ $t:= t+1$.
\STATE {\bfseries 11:} \textbf{End For}
\STATE {\bfseries Outlier Removal:} Remove columns corresponding to entry 0 in vector $\C$ from $\widehat{\B}^{s_0+r}=[\E^{s_0},\U^r]$.
\STATE {\bfseries Output:} Estimated range space, identified outlier vectors, and recovered underlying matrix $\widehat{\M}$ with column $\widehat{\M}_{:t}$.
\end{algorithmic}
\end{algorithm}

Regarding computational considerations, learning a $\tau$-sparse representation of a given vector w.r.t. a known dictionary can be done in polynomial time if the dictionary matrix satisfies the restricted isometry property~\cite{candes2006robust}, or trivially if $\tau$ is a constant~\cite{balcan2014efficient}. This can be done by applying $\ell_1$ minimization or brute-force algorithm, respectively. Indeed, many real datasets match the constant-$\tau$ assumption, e.g., face image~\cite{basri2003lambertian} (each person lies on a subspace of dimension $\tau=9$), 3D motion trajectory~\cite{Costeira} (each object lies on a subspace of dimension $\tau=4$), handwritten digits~\cite{hastie1998metrics} (each script lies on a subspace of dimension $\tau=12$), etc. So our algorithm is applicable for all these settings.

Theoretically, the following theorem provides a strong guarantee for our algorithm.
\begin{theorem}[Mixture of Subspaces]
\label{theorem: upper bound in noiseless case for mixture of subspaces}
Let $r$ be the rank of the underlying matrix $\L$. Suppose that the columns of $\L$ lie on a mixture of identifiable and independent subspaces, each of which is of dimension at most $\tau$. Denote by $\mu_\tau$ the maximal incoherence over all $\tau$-combinations of $\L$. Let the noise model be that of Theorem \ref{theorem: upper bound in noiseless case}. Then our algorithm exactly recovers the underlying matrix $\L$, the column space $\U^r$, and the outlier $\E^{s_0}$ with probability at least $1-\delta$, provided that $d\ge c\mu_\tau \tau^2\log\left(r/\delta\right)$ for some global constant $c$ and $s_0\le d-\tau-1$. The total sample complexity is thus $c\mu_\tau \tau^2 n\log\left(r/\delta\right)$.
\end{theorem}
As a concrete example, if the incoherence parameter $\mu_\tau$ is a global constant and the dimension $\tau$ of each subspace is far less than $r$, the sample complexity of $\O(\mu_\tau n\tau^2\log(r/\delta))$ is significantly better than the complexity of $\O(\mu_0 nr\log(r/\delta))$ for the structure of a single subspace in Theorem \ref{theorem: upper bound in noiseless case}. This argument shows that the sparse connections between the two layers improve the sample complexity.

\begin{proof}[Proof of Theorem \ref{theorem: upper bound in noiseless case for mixture of subspaces}]
We first study the effectiveness of our representation step.
\begin{lemma}
\label{lemma: outlier rank argument mixture of subspaces}
Let $[\E^s,\U^k]$ be the current dictionary matrix consisting of a random noise matrix $\E^s\in\R^{m\times s}$ and a clean basis matrix $\U^k\in\R^{m\times k}$. Suppose we get access to a set of coordinates $\Omega\subset [m]$ of size $d$ uniformly at random without replacement. Let $s\le d-\tau-1$ and $d\ge 8\mu_\tau \tau\log(\tau/\delta)$. Denote by $\U^\tau\in\R^{m\times \tau}$ a submatrix of $\U^k$ with $\tau$ columns.
\begin{itemize}
\item
If $\M_{:t}\in \U^r$ but it cannot be represented by a linear combination of $\tau$ vectors in the current dictionary, then with probability at least $1-\delta$, $\M_{\Omega t}$ does not belong to any fixed $\tau$-combination of the truncated dictionary as well.
\item
If $\M_{:t}$ can be represented by a linear combination of $\tau$ vectors in the current basis, then $\M_{\Omega t}$ can be represented as a linear combination of the same $\tau$ truncated vectors in the dictionary with probability $1$, the representation coefficients of $\M_{:t}$ corresponding to $\E^s$ in the dictionary is $\0$ with probability $1$, and $[\E^s,\U^k][\E_{\Omega:}^{s},\U_{\Omega:}^{k}]^{\dag}\M_{\Omega t}=\M_{: t}$ with probability at least $1-\delta$.
\item
If $\M_{:t}$ is an outlier drawn from a non-degenerate distribution, then $\M_{\Omega t}$ cannot be represented by the dictionary with probability $1$.
\end{itemize}
\end{lemma}

\begin{proof}
The proof is similar as that of Lemma \ref{lemma: outlier rank argument}. For completeness, we give a brief proof here. For the first part of the lemma, by Facts 2 and 4 of Lemma \ref{lemma: facts on non-degenerate distribution}, the $\E_{\Omega :}^s$ cannot have a non-zero representation coefficient of $\M_{\Omega t}$ in any possible $\tau$-combination of the current dictionary when $s\le d-\tau-1$, due to the randomness. Thus the problem of whether $\M_{\Omega t}$ can be $\tau$ represented by the current dictionary $[\E_{\Omega:}^s,\U_{\Omega:}^k]$ is totally determined by whether it can be $\tau$ represented by $\U_{\Omega:}^k$. Now suppose that $\M_{\Omega t}$ can be written as a linear $\tau$-combination of the current basis $\U_{\Omega :}^\tau$. Then according to Proposition \ref{proposition: subsampling with same rank}, since $d\ge 8\mu_\tau\tau\log(\tau/\delta)$, we have that $\rank([\U^\tau,\M_{:t}])=\tau$, which is contradictory with the assumption of Event 1.

The first argument in Event 2 is obvious. Now suppose that the representation coefficients of $\M_{:t}$ corresponding to $\E^s$ in the dictionary $[\E^s,\U^k]$ is NOT $\0$ and $\M_{:t}\in \U^k$. Then $\M_{:t}-\U^k\c\in\mathbf{span}(\E^s)$, where $\c$ is the representation coefficients of $\M_{:t}$ corresponding to $\U^k$ in the dictionary $[\E^s,\U^k]$. Also, note that $\M_{:t}-\U^k\c\in\U^k$. So $\rank[\E^s,\M_{:t}-\U^k\c]=s$, which is contradictory with Fact 2 of Lemma \ref{lemma: facts on non-degenerate distribution}. So the coefficient w.r.t. $\E^s$ in the dictionary $[\E^s,\U^k]$ is $\0$. Since by assumption $\M_{:t}$ can be represented by $\tau$ combination of columns in $\U^k$, termed $\U^\tau$, we have that $[\E^s,\U^k][\E_{\Omega:}^{s},\U_{\Omega:}^{k}]^{\dag}\M_{\Omega t}=\U^\tau\U_{\Omega:}^{\tau\dag}\M_{\Omega t}=\U^\tau(\U_{\Omega:}^{\tau T}\U_{\Omega:}^\tau)^{-1}\U_{\Omega:}^{\tau T}\M_{\Omega t}=\U^\tau(\U_{\Omega:}^{\tau T}\U_{\Omega:}^\tau)^{-1}\U_{\Omega:}^{\tau T}\U_{\Omega :}^\tau \v=\U^\tau \v=\M_{: t}$, where $\v$ is the representation coefficient of $\M_{:t}$ w.r.t. $\U^\tau$. (The $(\U_{\Omega:}^{\tau T}\U_{\Omega:}^\tau)^{-1}$ exists because $\rank(\U_{\Omega:}^\tau)=\tau$ by Proposition \ref{proposition: subsampling with same rank})

Event 3 is an immediate result of Lemma \ref{lemma: facts on non-degenerate distribution}.
\end{proof}

Now we are ready to prove Theorem \ref{theorem: upper bound in noiseless case for mixture of subspaces}. In fact, Theorem \ref{theorem: upper bound in noiseless case for mixture of subspaces} is a result of union bound of Lemma \ref{lemma: outlier rank argument mixture of subspaces}. For the event of type 1, the union bound is over $\binom{r}{\tau}=\O(r^\tau)$ events. For the event of type 2, since we resample $\Omega$ at most $r+s_0$ times by algorithm, the union bound is over $r+s_0$ samplings. The event of type 3 is with probability $1$. So overall, replacing $\delta$ with $\min\{\delta/r^\tau,\delta/(r+s_0)\}$ in Lemma \ref{lemma: outlier rank argument mixture of subspaces}, the sample complexity we need is at least $\O(\mu_\tau\tau\log(\max\{r^\tau,r+s_0\}/\delta))$. Note that $s_0\le d-\tau-1$. So the sample complexity for each column is at least $\O(\mu_\tau\tau^2\log(r/\delta))$ and the total one is $\O(\mu_\tau\tau^2n\log(r/\delta))$, as desired. The success of outlier removal step is guaranteed by Lemma \ref{lemma: outlier removal}.
\end{proof}

\section{Experimental Results}
\noindent \textbf{Bounded Deterministic Noise:} We verify the estimated error of our algorithm in Theorem \ref{theorem: noise tolerance bound} under bounded deterministic noise. Our synthetic data are generated as follows. We  construct $5$ base vectors $\{\u_i\}_{i=1}^5$ by sampling their entries from $\N(0,1)$. The underlying matrix $\L$ is then generated by
$\L=\left[\u_1\1_{200}^T,\sum_{i=1}^2\u_i\1_{200}^T,\sum_{i=1}^3\u_i\1_{200}^T,\sum_{i=1}^4\u_i\1_{200}^T,\sum_{i=1}^5\u_i\1_{1,200}^T\right]\in\R^{100\times 2,000}$, each column of which is normalized to the unit $\ell_2$ norm. Finally, we add bounded yet unstructured noise to each column, with noise level $\epsilon_{noise}=0.6$. We randomly pick $20\%$ entries to be unobserved. The left figure in Figure \ref{figure: synthetic data} shows the comparison between our estimated error\footnote{The estimated error is up to a constant factor.} and the true error by our algorithm. The result demonstrates that empirically, our estimated error successfully predicts the trend of the true algorithmic error.

\begin{figure}
\centering
\subfigure{
\includegraphics[width=0.38\textwidth]{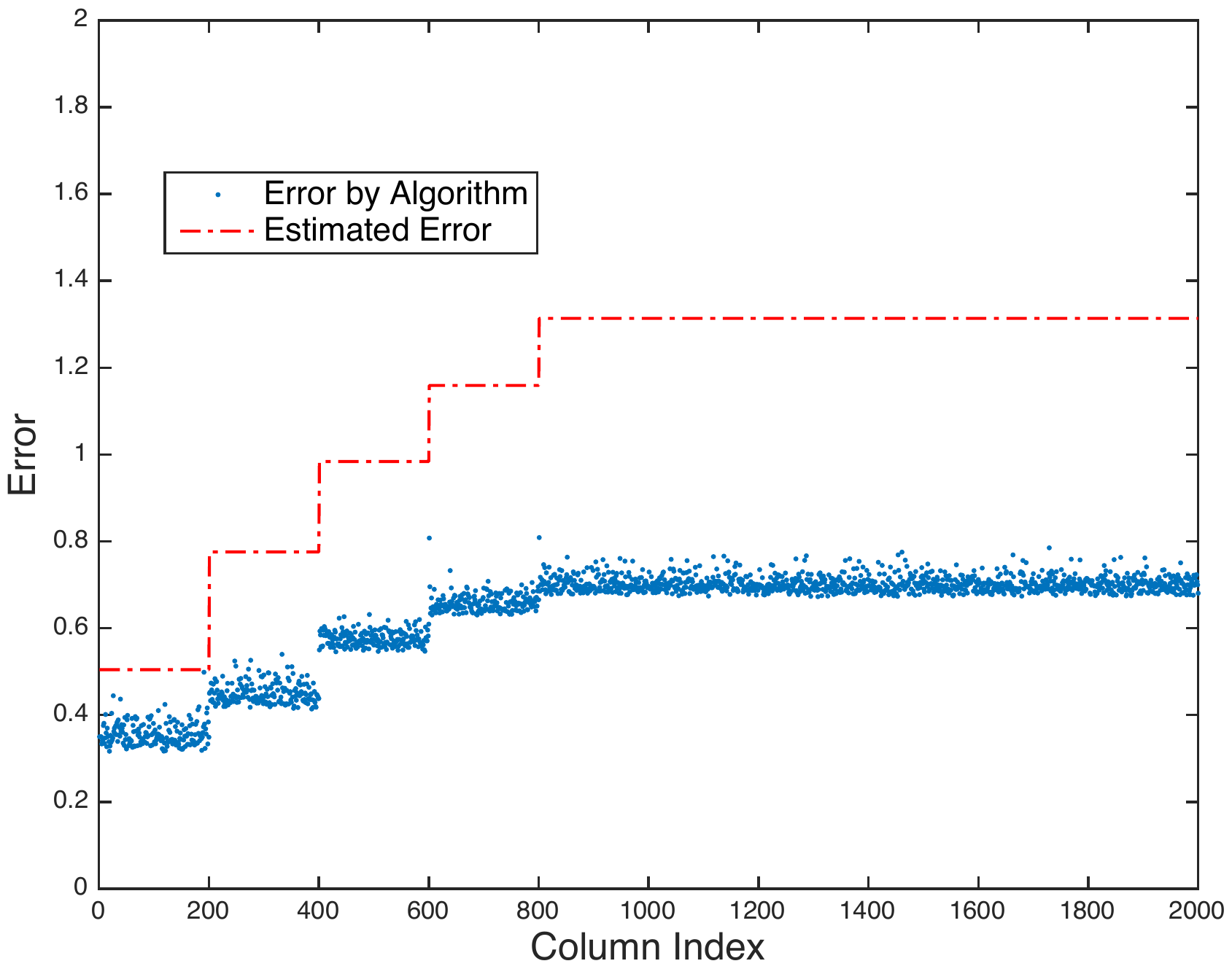}}
\subfigure{
\includegraphics[width=0.29\textwidth]{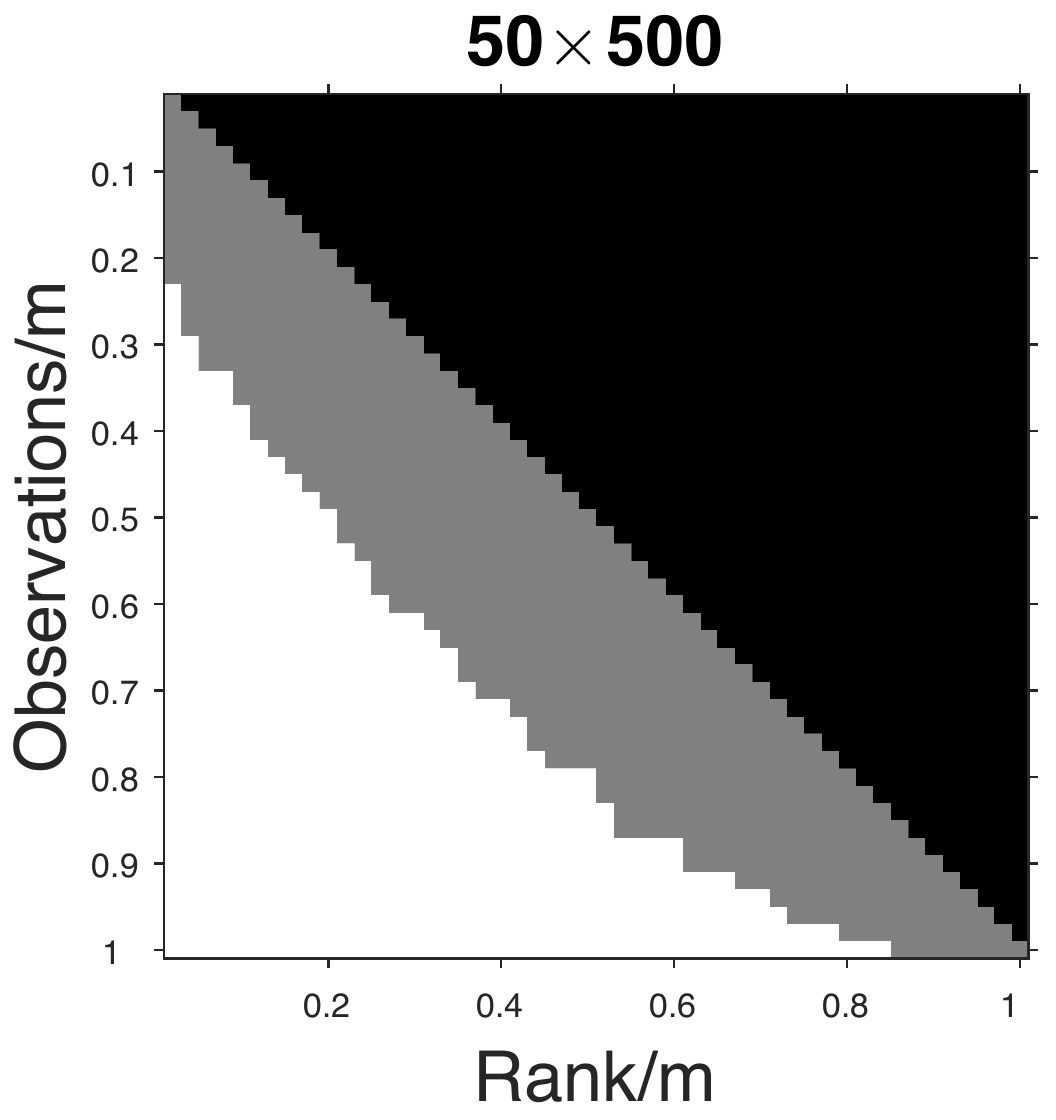}
}
\subfigure{
\includegraphics[width=0.29\textwidth]{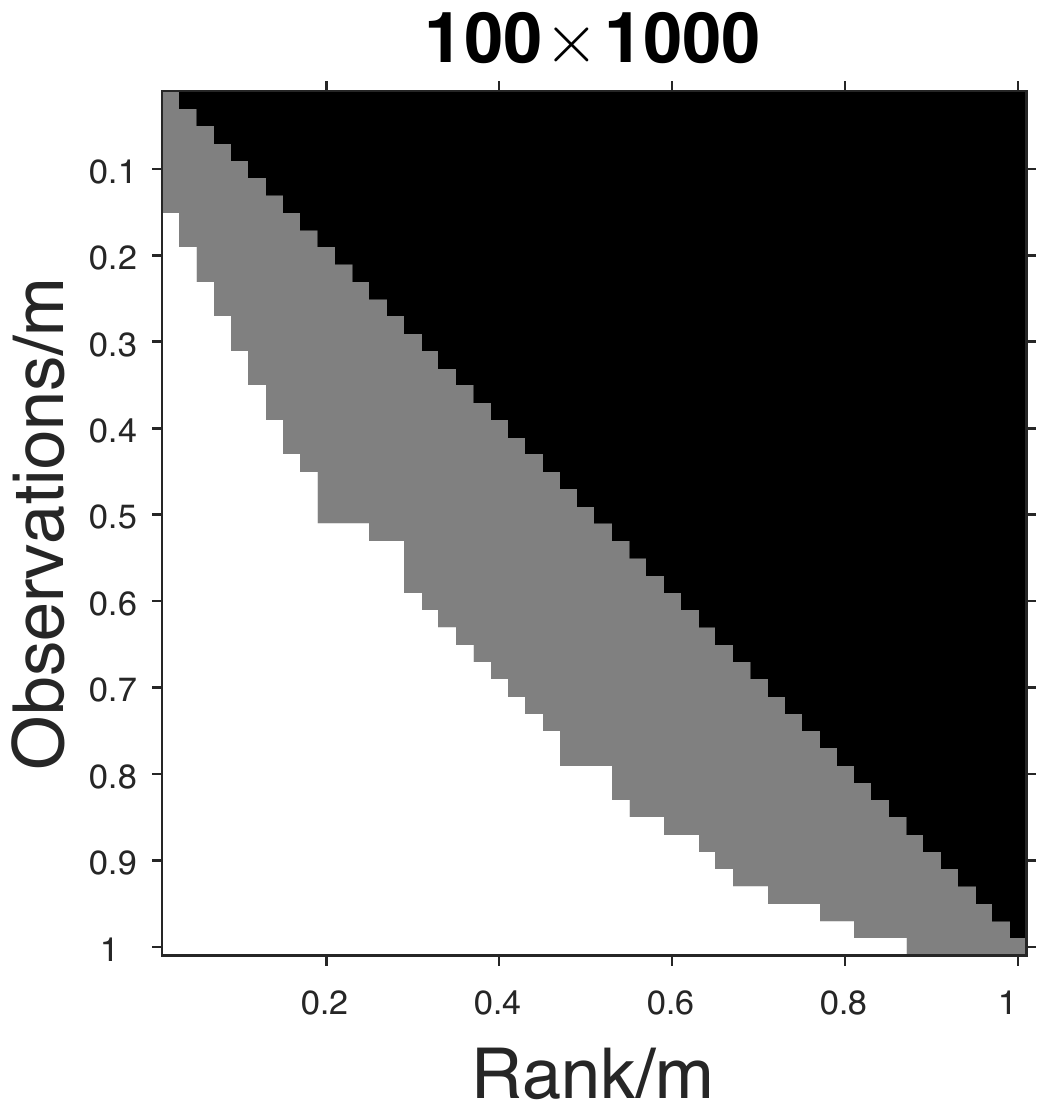}}
\caption{Left Figure: Approximate recovery under bounded deterministic noise with estimated error. Right Two Figures: Exact recovery under sparse random noise with varying rank and sample size. \textbf{White Region:} Nuclear norm minimization (passive sampling) succeeds. \textbf{White and Gray Regions:} Our algorithm (adaptive sampling) succeeds. \textbf{Black Region:} Our algorithm fails. It shows that the success region of our algorithm strictly contains that of the passive sampling method.}
\label{figure: synthetic data}
\end{figure}

\vspace{+0.4cm}
\noindent \textbf{Sparse Random Noise:} We then verify the exact recoverability of our algorithm under sparse random noise. The synthetic data are generated as follows. We construct the underlying matrix $\L=\X\Y$ as a product of $m\times r$ and $r\times n$ i.i.d. $\N(0,1)$ matrices. The sparse random noise is drawn from standard Gaussian distribution such that $s_0\le d-r-1$. For each size of problem ($50\times 500$ and $100\times 1,000$), we test with different rank ratios $r/m$ and measurement ratios $d/m$. The experiment is run by 10 times. We define that the algorithm succeeds if $\|\widehat{\L}-\L\|_F\le 10^{-6}$, $\rank(\widehat{\L})=r$, and the recovered support of the noise is exact for at least one experiment. The right two figures in Figure \ref{figure: synthetic data} plots the fraction of correct recoveries: white denotes perfect recovery by nuclear norm minimization approach \eqref{equ: outlier pursuit}; white+gray represents perfect recovery by our algorithm; black indicates failure for both methods. It shows that the success region of our algorithm strictly contains that of the prior approach. Moreover, the phase transition of our algorithm is nearly a linear function w.r.t $r$ and $d$. This is consistent with our prediction $d=\Omega(\mu_0 r\log(r/\delta))$ when $\delta$ is small, e.g., $\text{poly}(1/n)$.

\vspace{+0.4cm}
\noindent \textbf{Mixture of Subspaces:} To test the performance of our algorithm for the mixture of subspaces, we conduct an experiment on the Hopkins 155 dataset. The Hopkins 155 database is composed of $155$ matrices/tasks, each of which consists of multiple data points drawn from two or three motion objects. The trajectory of each object lie in a subspace. We input the data matrix to our algorithm with varying sample sizes. Table \ref{table:hopkins 155} records the average relative error $\|\widehat{\L}-\L\|_F/\|\L\|_F$ of 10 trials for the first five tasks in the dataset. It shows that our algorithm is able to recover the target matrix with high accuracy.
\begin{table}
\caption{Life-long Matrix Completion on the first 5 tasks in Hopkins 155 database.} \label{table:hopkins 155}
\vspace{-0.2cm}
\begin{center}
\begin{tabular}{c||c|c|c|c|c}
\hline
\#Task & Motion Number & $d=0.8m$ & $d=0.85m$ & $d=0.9m$ & $d=0.95m$\\
\hline
\hline
\#1 & 2 & $9.4\times10^{-3}$ & $6.0\times10^{-3}$ & $3.4\times10^{-3}$ & $2.6\times10^{-3}$\\
\hline
\#2 & 3 & $5.9\times10^{-3}$ & $4.4\times10^{-3}$ & $2.4\times10^{-3}$ & $1.9\times10^{-3}$\\
\hline
\#3 & 2 & $6.3\times10^{-3}$ & $4.8\times10^{-3}$ & $2.8\times10^{-3}$ & $7.2\times10^{-4}$\\
\hline
\#4 & 2 & $7.1\times10^{-3}$ & $6.8\times10^{-3}$ &  $6.1\times10^{-3}$ & $1.5\times10^{-3}$\\
\hline
\#5 & 2 & $8.7\times10^{-3}$ & $5.8\times10^{-3}$ & $3.1\times10^{-3}$ & $1.2\times10^{-3}$\\
\hline
\end{tabular}
\end{center}
\end{table}

\noindent \textbf{Single Subspace v.s. Mixture of Subspaces:}
We compare the sample complexity of Algorithm \ref{algorithm: never ending learning sparse noise} (Single Subspace) and Algorithm \ref{algorithm: never ending learning sparse noise mixture of subspaces} (Mixture of Subspaces) for the exact recovery of the underlying matrix. The data are generated as follows. We construct 5 independent subspaces $\{\mathcal{S}_i\}_{i=1}^5$ whose bases $\{\U_i\}_{i=1}^5$ are $100\times 4$ random matrices consisting of orthogonal columns ($\tau=4$ and $r=20$). We then sample 20 data from each subspace uniformly and obtain a $100\times 100$ data matrix. The sample size $d$ varies from 1 to 100, and we record the empirical probability of success over $200$ times of experiments, where we define that an algorithm succeeds if $\|\widehat{\L}-\L\|_F\le 10^{-6}$ and $\rank(\widehat{\L})=r$. As shown in Figure \ref{figure: single_vs_mixture}, we see that the sample complexity can indeed be smaller in the case of mixture of subspaces.
\begin{figure}[ht]
\centering
\includegraphics[width=0.5\textwidth]{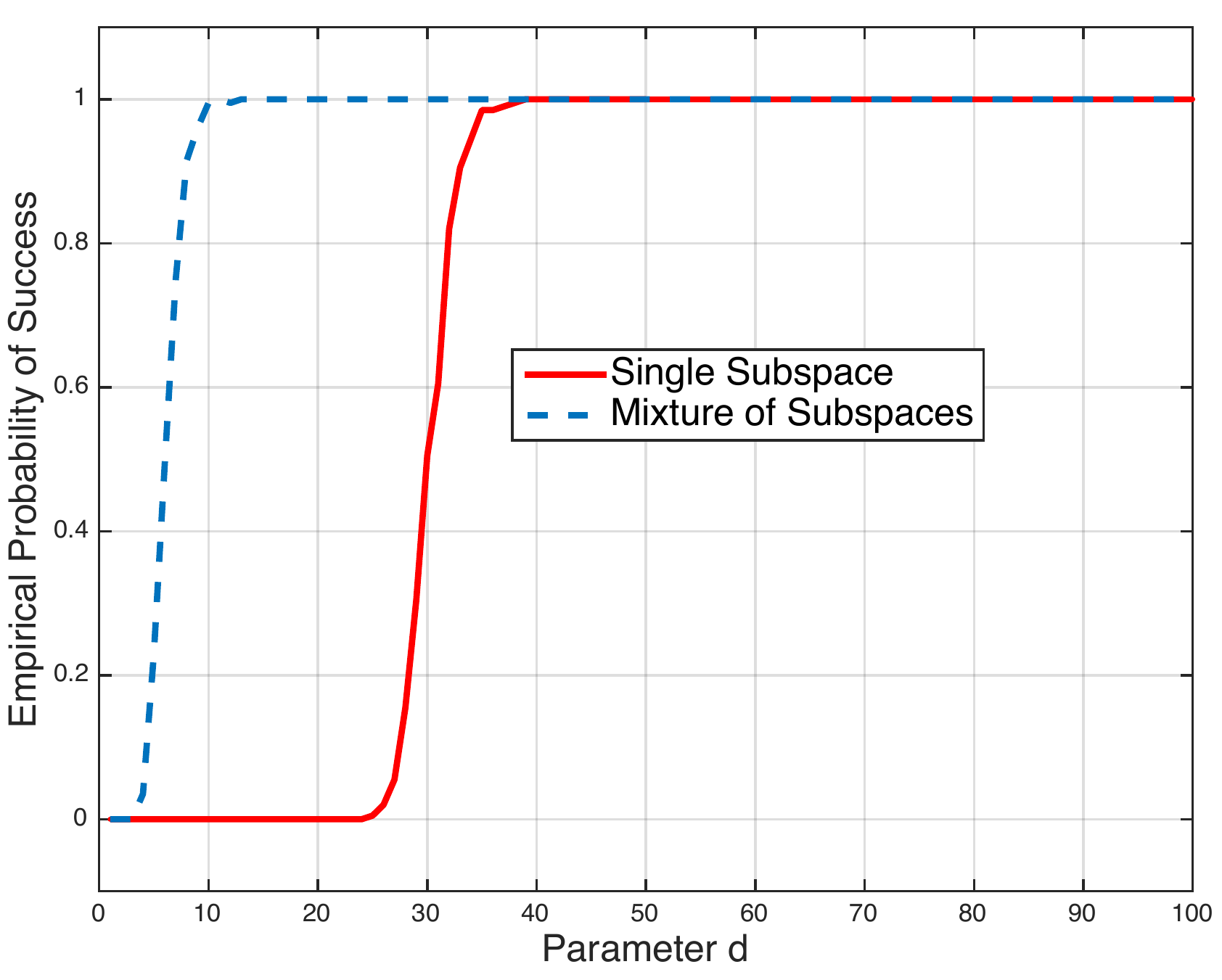}
\caption{Comparison of sample complexity between the case of single subspace and that of mixture of subspaces.}
\label{figure: single_vs_mixture}
\end{figure}

\section{Conclusions}

In this paper, we study life-long matrix completion that aims at online recovering an $m\times n$ matrix of rank $r$ under two realistic noise models --- bounded deterministic noise and sparse random noise. Our result advances the state-of-the-art work and matches the lower bound under sparse random noise. In a more benign setting where the columns of the underlying matrix lie on a mixture of subspaces, we show that a smaller sample complexity is possible to exactly recover the target matrix. It would be interesting to extend our results to other realistic noise
models, including random classification noise or malicious noise previously studied in the context of supervised
classification~\cite{awasthi2014power,balcan2013statistical}

\noindent \textbf{Acknowledgements.} This work was supported in part by NSF grants NSF CCF-1422910, NSF
CCF-1535967,  NSF CCF-1451177, NSF IIS-1618714, a Sloan Research
Fellowship, and a Microsoft Research Faculty Fellowship.

\newpage
\bibliographystyle{alpha}
\bibliography{reference}

\newpage
\appendix

\comment{
\section{Supplementary Experiments}
\subsection{Single Subspace v.s. Mixture of Subspaces}
We compare the sample complexity of Algorithm \ref{algorithm: never ending learning sparse noise} (Single Subspace) and Algorithm \ref{algorithm: never ending learning sparse noise mixture of subspaces} (Mixture of Subspaces) for the exact recovery of the underlying matrix. The data are generated as follows. We construct 5 independent subspaces $\{\mathcal{S}_i\}_{i=1}^5$ whose bases $\{\U_i\}_{i=1}^5$ are $100\times 4$ random matrices consisting of orthogonal columns ($\tau=4$ and $r=20$). We then sample 20 data from each subspace uniformly and obtain a $100\times 100$ data matrix. The sample size $d$ varies from 1 to 100, and we record the empirical probability of success over $200$ times of experiments, where we define that an algorithm succeeds if $\|\widehat{\L}-\L\|_F\le 10^{-6}$ and $\rank(\widehat{\L})=r$. As shown in Figure \ref{figure: single_vs_mixture}, we see that the sample complexity can indeed be smaller in the case of mixture of subspaces.
\begin{figure}[ht]
\centering
\includegraphics[width=0.5\textwidth]{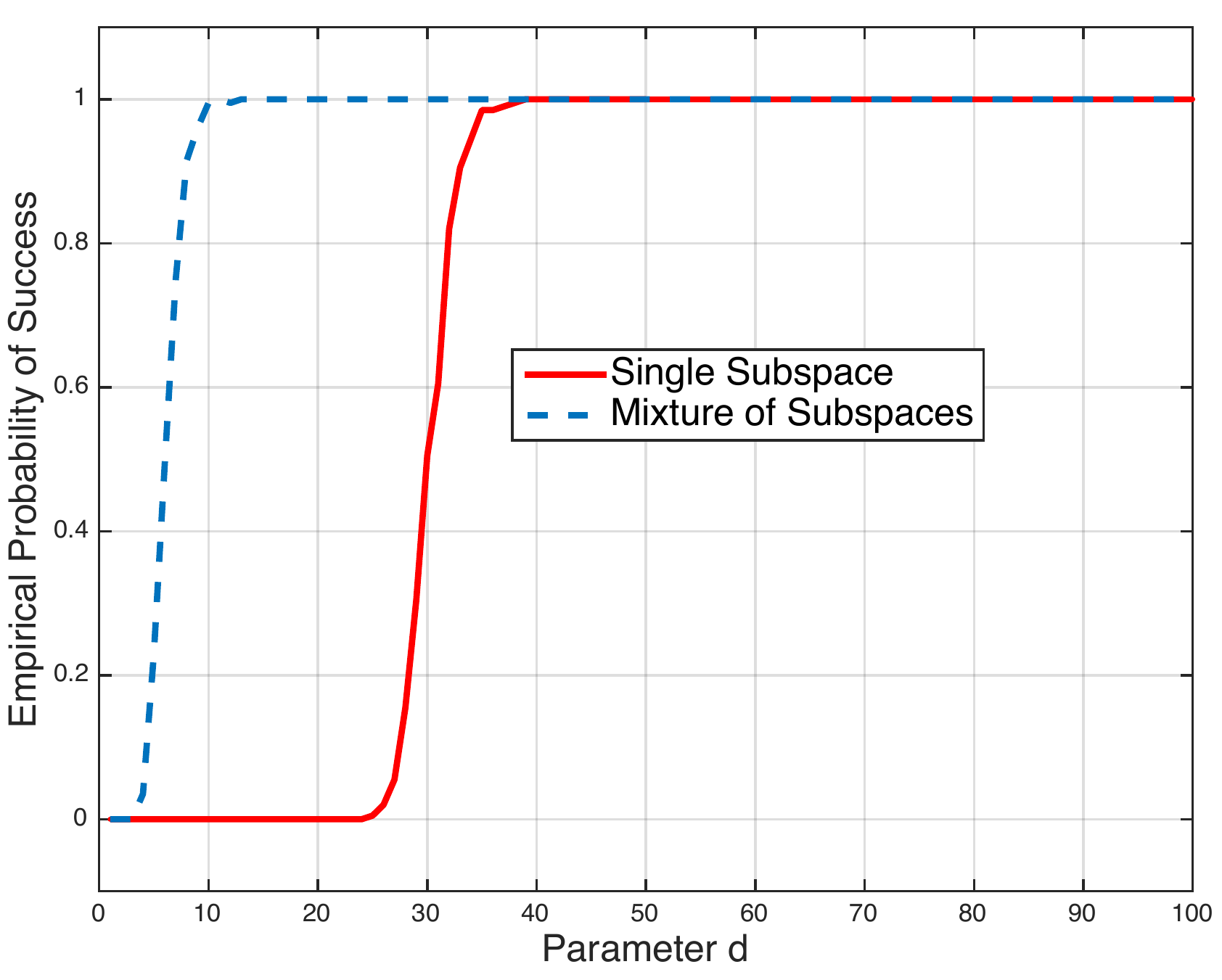}
\caption{Comparison of sample complexity between the case of single subspace and that of mixture of subspaces.}
\label{figure: single_vs_mixture}
\end{figure}
}

\comment{
\section{Proof of Robust Recovery under Deterministic Noise}
\begin{lemma}[Lemma 2.~\cite{balcan2014efficient}]
\label{lemma: subspaces angle diff one vect}
Let $\mathbf{W}=\mathbf{span}\{\w_1,\w_2,...,\w_{k-1}\}$, $\U=\mathbf{span}\{\w_1,\w_2,...,\w_{k-1},\u\}$, and $\widetilde{\U}=\mathbf{span}\{\w_1,\w_2,...,\w_{k-1},\widetilde{\u}\}$ be subspaces spanned by vectors in $\R^m$. Then
\begin{equation*}
\theta\left(\U,\widetilde{\U}\right)\le \frac{\pi}{2}\frac{\theta(\widetilde{\u},\u)}{\theta(\widetilde{\u},\mathbf{W})}.
\end{equation*}
\end{lemma}

\noindent{\textbf{Lemma ~\ref{lemma: angle between U and tilde U} (restated).}}
\emph{Let $\U^k=\mathbf{span}\{\u_1,\u_2,...,\u_k\}$ and $\widetilde{\U}^k=\mathbf{span}\{\widetilde{\u}_1,\widetilde{\u}_2,...,\widetilde{\u}_k\}$ be two subspaces such that $\theta(\u_i,\widetilde{\u}_i)\le\epsilon_{noise}$ for all $i\in[k]$. Let $\gamma_k=\sqrt{20k\epsilon_{noise}}$ and $\theta(\widetilde{\u}_i,\widetilde{\U}^{i-1})\ge \gamma_i$ for $i=2,...,k$. Then
\begin{equation}
\label{equ: error propagation}
\theta\left(\U^k,\widetilde{\U}^k\right)\le 10k\frac{\epsilon_{noise}}{\gamma_k}=\frac{\gamma_k}{2}.
\end{equation}}

\begin{proof}
The proof is basically by induction on $k$. Instead, we will prove a stronger result by showing that \eqref{equ: error propagation} holds on subspaces $\U^k=\mathbf{span}\{\mathbf{W},\u_1,\u_2,...,\u_k\}$ and $\widetilde{\U}^k=\mathbf{span}\{\mathbf{W},\widetilde{\u}_1,\widetilde{\u}_2,...,\widetilde{\u}_k\}$ for arbitrary fixed subspace $\mathbf{W}$. The base case $k=1$ follows immediately from Lemma \ref{lemma: subspaces angle diff one vect}. Now suppose the conclusion holds for any index $\le k-1$. Let $\U_0^k=\mathbf{span}\{\U^{k-1},\widetilde{\u}_k\}$. Then for the index $k$, we have
\begin{equation*}
\begin{split}
\theta(\U^k,\widetilde{\U}^k)&\le \theta(\U^k,\U_0^k)+\theta(\U_0^k,\widetilde{\U}^k)\\
&\le \frac{\pi}{2}\frac{\theta(\widetilde{\u}_k,\u_k)}{\theta(\widetilde{\u}_k,\U^{k-1})}+10(k-1)\frac{\epsilon_{noise}}{\gamma_{k-1}}\ \ (\text{By Lemma \ref{lemma: subspaces angle diff one vect} and induction hypothesis})\\
&\le \frac{\pi}{2}\frac{\epsilon_{noise}}{\theta(\widetilde{\u}_k,\widetilde{\U}^{k-1})-\theta(\U^{k-1},\widetilde{\U}^{k-1})}+10(k-1)\frac{\epsilon_{noise}}{\gamma_{k-1}}\\
&\le \frac{\pi}{2}\frac{\epsilon_{noise}}{\gamma_k-10(k-1)\frac{\epsilon_{noise}}{\gamma_{k-1}}}+10(k-1)\frac{\epsilon_{noise}}{\gamma_{k-1}}\ \ (\text{By induction hypothesis})\\
&= \frac{\epsilon_{noise}}{\gamma_{k-1}}\left(\frac{\pi}{2}\frac{\gamma_{k-1}^2}{\gamma_k\gamma_{k-1}-10(k-1)\epsilon_{noise}}+10(k-1)\right)\\
&\le \frac{\epsilon_{noise}}{\gamma_{k-1}}(\pi+10k-10)\\
&= \frac{\epsilon_{noise}}{\gamma_{k}}\frac{\gamma_k}{\gamma_{k-1}}(\pi+10k-10)\\
&= \frac{\epsilon_{noise}}{\gamma_{k}}\sqrt{\frac{k}{k-1}}(\pi+10k-10)\\
&\le \frac{\epsilon_{noise}}{\gamma_{k}}10k\ \ \ \ (k\ge 2).
\end{split}
\end{equation*}
\end{proof}
}

\comment{
\noindent{\textbf{Lemma ~\ref{lemma: residual estimate} (restated).}}
\emph{Let $\epsilon_k=2\gamma_k$, $\gamma_k=\sqrt{20k\epsilon_{noise}}$, and $k\le r$. Suppose that we observe a set of coordinates $\Omega\subset [m]$ of size $d$ uniformly at random with replacement, where $d\ge c_0(\mu_0r+mk\epsilon_{noise})\log^2 (2/\delta)$. If $\theta(\L_{:t},\widetilde{\U}^k)\le \epsilon_k$, then with probability at least $1-4\delta$, we have
\begin{equation*}
\|\M_{\Omega t}-\P_{\widetilde{\U}_{\Omega:}^k}\M_{\Omega t}\|_2\le C\sqrt{\frac{dk\epsilon_{noise}}{m}}.
\end{equation*}
Inversely, if $\theta(\L_{:t},\widetilde{\U}^k)\ge c\epsilon_k$, then with probability at least $1-4\delta$, we have
\begin{equation*}
\|\M_{\Omega t}-\P_{\widetilde{\U}_{\Omega:}^k}\M_{\Omega t}\|_2\ge C\sqrt{\frac{dk\epsilon_{noise}}{m}},
\end{equation*}
where $c_0$, $c$ and $C$ are absolute constants.}

\begin{proof}
The first part of the theorem follows from the upper bound of Lemma \ref{lemma: concentration of measure}. Specifically, by plugging $d$ into the lower bound of Lemma \ref{lemma: concentration of measure}, we see that $\alpha<1/2$ and $\gamma<1/3$. Note that
\begin{equation*}
\left\|\L_{:t}-\P_{\widetilde{\U}^k}\L_{:t}\right\|_2=\left\|\L_{:t}\right\|_2\sin \theta\left(\L_{:t},\P_{\widetilde{\U}^k}\L_{:t}\right)\le \theta\left(\L_{:t},\P_{\widetilde{\U}^k}\L_{:t}\right)=\theta(\L_{:t},\widetilde{\U}^k)\le \epsilon_k.
\end{equation*}
Therefore, by Lemma \ref{lemma: concentration of measure},
\begin{equation*}
\begin{split}
\left\|\M_{\Omega t}-\P_{\widetilde{\U}_{\Omega :}^k}\M_{\Omega t}\right\|_2&\le \O\left(\sqrt{\frac{d}{m}}\left\|\M_{:t}-\P_{\widetilde{\U}^k}\M_{:t}\right\|_2\right)\\
&\le \O\left(\sqrt{\frac{d}{m}}\left(\left\|\L_{:t}-\P_{\widetilde{\U}^k}\L_{:t}\right\|_2+\left\|\P_{\widetilde{\U}^k}(\L_{:t}-\M_{:t})\right\|_2+\left\|\M_{:t}-\L_{:t}\right\|_2\right)\right)\\
&\le \O\left(\sqrt{\frac{d}{m}}(\epsilon_k+2\epsilon_{noise})\right)\\
&\le C\sqrt{\frac{dk\epsilon_{noise}}{m}}.
\end{split}
\end{equation*}

We now proceed the second part of the theorem. To this end, we first explore the relation between the incoherence of the noisy basis $\widetilde{\U}^k$ and the clean one $\U^k$. Since we are able to control the error propagation in $\widetilde{\U}^k$, intuitively, the incoherence of $\widetilde{\U}^k$ and $\U^k$ is not distinct too much. In particular, for any $i\in [m]$,
\begin{equation*}
\begin{split}
\left\|\P_{\widetilde{\U}^k}\e_i\right\|_2&\le \left\|\P_{\U^k}\e_i\right\|_2+\left\|\P_{\U^k}\e_i-\P_{\widetilde{\U}^k}\e_i\right\|_2\\
&\le \left\|\P_{\U^k}\e_i\right\|_2+\left\|\P_{\U^k}-\P_{\widetilde{\U}^k}\right\|\left\|\e_i\right\|_2\\
&= \left\|\P_{\U^k}\e_i\right\|_2+\left\|\e_i\right\|_2\sin\theta\left(\U^k,\widetilde{\U}^k\right)\\
&\le \left\|\P_{\U^k}\e_i\right\|_2+\theta\left(\U^k,\widetilde{\U}^k\right)\\
&\le \left\|\P_{\U^k}\e_i\right\|_2+\frac{\gamma_k}{2}\\
&= \left\|\P_{\U^k}\e_i\right\|_2+\frac{1}{4}\epsilon_k.
\end{split}
\end{equation*}
Therefore,
\begin{equation*}
\mu\left(\widetilde{\U}^k\right)=\frac{m}{k}\max_{i\in[m]}\left\|\P_{\widetilde{\U}^k}\e_i\right\|_2^2\le \frac{m}{k}\left(2\left\|\P_{\U^k}\e_i\right\|_2^2+\frac{1}{8}\epsilon_k^2\right)\le 2\mu(\U^k)+c''m\epsilon_{noise},
\end{equation*}
for global constant $c''$. Also, note that
\begin{equation*}
\left\|\P_{\widetilde{\U}^k}\M_{:t}-\L_{:t}\right\|_2\ge \sin\theta\left(\L_{:t},\P_{\widetilde{\U}^k}\M_{:t}\right)\|\L_{:t}\|_2\ge \frac{1}{2}\theta\left(\L_{:t},\P_{\widetilde{\U}^k}\M_{:t}\right)\ge \frac{1}{2}\theta\left(\L_{:t},\widetilde{\U}^k\right)\ge \frac{c\epsilon_k}{2}.
\end{equation*}
So we have
\begin{equation*}
\begin{split}
&\left\|\M_{\Omega t}-\P_{\widetilde{\U}_\Omega^k}\M_{\Omega t}\right\|_2\ge \sqrt{\frac{1}{m}\left(\frac{d}{2}-\frac{3k\mu(\widetilde{\U}^k)\beta}{2}\right)}\left\|\M_{:t}-\P_{\widetilde{\U}^k}\M_{:t}\right\|_2\\
&\ge \Omega\left(\sqrt{\frac{d}{m}-\frac{3k\mu(\U^k)}{m}\log^2 (1/\delta)-c_0k\epsilon_{noise}\log^2 (1/\delta)}\left\|\M_{:t}-\P_{\widetilde{\U}^k}\M_{:t}\right\|_2\right)\\
&\ge \Omega\left(\sqrt{\frac{d}{m}-\frac{3\mu_0r}{m}\log^2 (1/\delta)-c_0k\epsilon_{noise}\log^2 (1/\delta)}\left\|\M_{:t}-\P_{\widetilde{\U}^k}\M_{:t}\right\|_2\right)\ \ (\mbox{Since } \U^k\subseteq\U^r)\\
&\ge \Omega\left(\sqrt{\frac{d}{m}-c_0k\epsilon_{noise}\log^2 (1/\delta)}\left(\left\|\P_{\widetilde{\U}^k}\M_{:t}-\L_{:t}\right\|_2-\left\|\L_{:t}-\M_{:t}\right\|_2\right)\right)\ \ \left(\mbox{Since }d>3\mu_0r\log^2(1/\delta)\right)\\
&> \Omega\left(\sqrt{\frac{d}{m}-c_0k\epsilon_{noise}\log^2 (1/\delta)}\left(\frac{c\epsilon_k}{2}-\epsilon_{noise}\right)\right)\\
&> C\sqrt{\frac{dk\epsilon_{noise}}{m}}\ \ \left(\mbox{Since }d> c_0mk\epsilon_{noise}\log^2(1/\delta)\right).
\end{split}
\end{equation*}
\end{proof}
}

\comment{
\begin{algorithm}[ht]
\caption{Noise-Tolerant Life-Long Matrix Completion under Bounded Deterministic Noise}
\begin{algorithmic}
\label{algorithm: never ending learning}
\STATE {\bfseries Input:} Columns of matrices arriving over time.
\STATE {\bfseries Initialize:} Let the basis matrix $\widehat{\U}^0=\emptyset$. Randomly draw entries $\Omega\subset [m]$ of size $d$ uniformly with replacement.
\STATE {\bfseries 1:} \textbf{For} $t$ from $1$ to $n$, \textbf{do}
\STATE {\bfseries 2:}\ \ \ \ \ \ \ \ (a) If $\|\M_{\Omega t}-\P_{\widehat{\U}_{\Omega:}^k}\M_{\Omega t}\|_2>\eta_k$
\vspace{-0.1cm}
\STATE {\bfseries 3:}\ \ \ \ \ \ \ \ \ \ \ \ \ i. Fully measure $\M_{:t}$ and add it to the basis matrix $\widehat{\U}^k$. Orthogonalize $\widehat{\U}^k$.
\STATE {\bfseries 4:}\ \ \ \ \ \ \ \ \ \ \ \ \ ii. Randomly draw entries $\Omega\subset [m]$ of size $d$ uniformly with replacement.
\STATE {\bfseries 5:}\ \ \ \ \ \ \ \ \ \ \ \ \ iii. $k:=k+1$.
\STATE {\bfseries 6:}\ \ \ \ \ \ \ \ (b) Otherwise $\widehat{\M}_{:t}:=\widehat{\U}^k\widehat{\U}_{\Omega:}^{k\dag}\M_{\Omega t}$.
\STATE {\bfseries 7:} \textbf{End For}
\STATE {\bfseries Output:} Estimated range space $\widehat{\U}^K$ and the underlying matrix $\widehat{\M}$ with column $\widehat{\M}_{:t}$.
\end{algorithmic}
\end{algorithm}
}

\comment{
\noindent{\textbf{Theorem ~\ref{theorem: noise tolerance bound} (restated).}}
\emph{Let $r$ be the rank of the underlying matrix $\L$ with $\mu_0$-incoherent column space. Suppose that the $\ell_2$ norm of the noise in each column is upper bounded by $\epsilon_{noise}$. Set $d\ge c_0(\mu_0r+mk\epsilon_{noise})\log^2 (2n/\delta))$ and
$\eta_k=C\sqrt{dk\epsilon_{noise}/m}$ for some global constants $c_0$ and $C$. Then Algorithm \ref{algorithm: never ending learning} outputs $\widehat{\U}^K$ with $K\le r$, $\widehat{\M}$ with $\ell_2$ error $\|\widehat{\M}_{:t}-\L_{:t}\|_2\le \Theta\left(\frac{m}{d}\sqrt{k\epsilon_{noise}}\right)$ uniformly for all $t$ with probability at least $1-\delta$, where $k\le r$ is the number of base vectors when learning the $t$-th column.}

\begin{proof}
We first show $K\le r$. Every time we add a new direction to the basis matrix if and only if Condition (a) in Algorithm \ref{algorithm: never ending learning} holds true. In that case by Lemma \ref{lemma: residual estimate}, if setting $\eta_k=C\sqrt{dk\epsilon_{noise}/m}$, then with probability at least $1-4\delta$, we have that $\theta(\L_{:t},\widetilde{\U}^k)\ge 2\gamma_k$, which implies $\theta(\M_{:t},\widetilde{\U}^k)\ge\theta(\L_{:t},\widetilde{\U}^k)-\theta(\M_{:t},\L_{:t}) \ge \gamma_k$. So by Lemma \ref{lemma: angle between U and tilde U}, $\theta(\U^k,\widetilde{\U}^k)\le\gamma_k/2$. Thus $\theta(\L_{:t},\U^k)\ge \theta(\L_{:t},\widetilde{\U}^k)-\theta(\U^k,\widetilde{\U}^k)\ge 3\gamma_k/2$. Since $\rank(\L)=r$, we obtain that $K\le r$.

We now proceed to prove the upper bound on the $\ell_2$ error. We discuss Case (a) and (b) respectively. If Condition (a) in Algorithm \ref{algorithm: never ending learning} holds true, then according to the algorithm, we fully observe $\M_{:t}$ and use it as our estimate $\widehat{\M}_{:t}$. So $\left\|\widehat{\M}_{:t}-\L_{:t}\right\|_2\le \epsilon_{noise}\le\Theta(\frac{m}{d}\sqrt{k\epsilon_{noise}})$; On the other hand, if Case (b) in Algorithm \ref{algorithm: never ending learning} holds true, then we represent $\widehat{\M}_{:t}$ by the basis subspace $\widetilde{\U}^k$. So we have
\begin{equation*}
\begin{split}
&\ \ \ \ \ \left\|\widehat{\M}_{:t}-\L_{:t}\right\|_2=\left\|\widetilde{\U}^k\widetilde{\U}_{\Omega:}^{k\dag}\M_{\Omega t}-\L_{:t}\right\|_2\\
&\le \left\|\widetilde{\U}^k\widetilde{\U}^{k\dag}\L_{: t}-\L_{:t}\right\|_2+\left\|\widetilde{\U}^k\widetilde{\U}_{\Omega:}^{k\dag}\L_{\Omega t}-\widetilde{\U}^k\widetilde{\U}^{k\dag}\L_{: t}\right\|_2+\left\|\widetilde{\U}^k\widetilde{\U}_{\Omega:}^{k\dag}\L_{\Omega t}-\widetilde{\U}^k\widetilde{\U}_{\Omega:}^{k\dag}\M_{\Omega t}\right\|_2\\
& = \sin\theta(\L_{:t},\widetilde{\U}^k)+\left\|\widetilde{\U}^k\widetilde{\U}_{\Omega:}^{k\dag}\L_{\Omega t}-\widetilde{\U}^k\widetilde{\U}^{k\dag}\L_{: t}\right\|_2+\left\|\widetilde{\U}^k\widetilde{\U}_{\Omega :}^{k\dag}(\L_{\Omega t}-\M_{\Omega t})\right\|_2.
\end{split}
\end{equation*}
To bound the second term, let $\L_{:t}=\widetilde{\U}^k\v+\e$, where $\widetilde{\U}^k\v=\widetilde{\U}^k\widetilde{\U}^{k\dag}\L_{:t}$ and $\|\e\|_2\le\epsilon_k$ since $\|\e\|_2=\sin\theta(\L_{:t},\widetilde{\U}^k)\le \epsilon_k$. So
\begin{equation*}
\begin{split}
\widetilde{\U}^k\widetilde{\U}_{\Omega:}^{k\dag}\L_{\Omega t}-\widetilde{\U}^k\widetilde{\U}^{k\dag}\L_{: t}&=\widetilde{\U}^k(\widetilde{\U}_{\Omega:}^{kT}\widetilde{\U}_{\Omega:}^{k})^{-1}\widetilde{\U}_{\Omega:}^{kT}(\widetilde{\U}_{\Omega:}^{k}\v+\e_{\Omega})-\widetilde{\U}_{\Omega:}^{k}\v\\
&=\widetilde{\U}^k\widetilde{\U}_{\Omega:}^{k\dag}\e_{\Omega}.
\end{split}
\end{equation*}
Therefore,
\begin{equation*}
\begin{split}
\left\|\widehat{\M}_{:t}-\L_{:t}\right\|_2&\le \theta(\L_{:t},\widetilde{\U}^k)+\left\|\widetilde{\U}^k\widetilde{\U}_{\Omega:}^{k\dag}\L_{\Omega t}-\widetilde{\U}^k\widetilde{\U}^{k\dag}\L_{: t}\right\|_2+\left\|\widetilde{\U}^k\widetilde{\U}_{\Omega:}^{k\dag}(\L_{\Omega t}-\M_{\Omega t})\right\|_2\\
&\le \theta(\L_{:t},\widetilde{\U}^k)+\left\|\widetilde{\U}^k\widetilde{\U}_{\Omega:}^{k\dag}\right\|\left\|\e_{\Omega}\right\|_2+\left\|\widetilde{\U}^k\widetilde{\U}_{\Omega:}^{k\dag}\right\|\left\|\L_{\Omega t}-\M_{\Omega t}\right\|_2\\
&\le \epsilon_k+\Theta\left(\frac{m}{d}\epsilon_k\right)+\Theta\left(\frac{m}{d}\epsilon_{noise}\right)\\
&=\Theta\left(\frac{m}{d}\sqrt{k\epsilon_{noise}}\right),
\end{split}
\end{equation*}
where $\left\|\widetilde{\U}^k\widetilde{\U}_{\Omega :}^k\right\|\le\sigma_1(\widetilde{\U}^k)/\sigma_k(\widetilde{\U}_{\Omega:}^k)\le\Theta(m/d)$ once $d\ge \Omega(\mu(\widetilde{\U}^k)k\log (k/\delta))$, due to Lemma \ref{lemma: matrix chernoff bound}.
The final sample complexity follows from the union bound on the $n$ columns.
\end{proof}
}

\comment{
\section{Proof of Exact Recovery under Random Noise}
\begin{algorithm}
\caption{Noise-Tolerant Life-Long Matrix Completion under Sparse Random Noise}
\begin{algorithmic}
\label{algorithm: never ending learning sparse noise}
\STATE {\bfseries Input:} Columns of matrices arriving over time.
\STATE {\bfseries Initialize:} Let the basis matrix $\widehat{\B}^0=\emptyset$, the counter $\C=\emptyset$. Randomly draw entries $\Omega\subset [m]$ of size $d$ uniformly without replacement.
\STATE {\bfseries 1:} \textbf{For} each column $t$ of $\M$, \textbf{do}
\STATE {\bfseries 2:}\ \ \ \ \ \ \ \ (a) If $\|\M_{\Omega t}-\P_{\widehat{\B}_{\Omega:}^k}\M_{\Omega t}\|_2>0$
\STATE {\bfseries 3:}\ \ \ \ \ \ \ \ \ \ \ \ \ i. Fully measure $\M_{:t}$ and add it to the basis matrix $\widehat{\B}^k$.
\STATE {\bfseries 4:}\ \ \ \ \ \ \ \ \ \ \ \ \ ii. $\C:=[\C,0]$.
\STATE {\bfseries 5:}\ \ \ \ \ \ \ \ \ \ \ \ \ iii. Randomly draw entries $\Omega\subset [m]$ of size $d$ uniformly without replacement.
\STATE {\bfseries 6:}\ \ \ \ \ \ \ \ \ \ \ \ \ iv. $k:=k+1$.
\STATE {\bfseries 7:}\ \ \ \ \ \ \ \ (b) Otherwise
\STATE {\bfseries 8:}\ \ \ \ \ \ \ \ \ \ \ \ \ i. $\C:=\C+\1_{supp(\widehat{\B}_{\Omega:}^{k\dag}\M_{\Omega t})}^T$.\ \ \ \ //Record supports of representation coefficient
\STATE {\bfseries 9:}\ \ \ \ \ \ \ \ \ \ \ \ \ ii. $\widehat{\M}_{:t}:=\widehat{\B}^k\widehat{\B}_{\Omega:}^{k\dag}\M_{\Omega t}$.
\STATE {\bfseries 10:}\ \ \ \ \ \ \ $t:= t+1$.
\STATE {\bfseries 11:} \textbf{End For}
\STATE {\bfseries Outlier Removal:} Remove columns corresponding to entry 0 in vector $\C$ from $\widehat{\B}^{s_0+r}=[\E^{s_0},\U^r]$.
\STATE {\bfseries Output:} Estimated range space, identified outlier vectors, and recovered underlying matrix $\widehat{\M}$ with column $\widehat{\M}_{:t}$.
\end{algorithmic}
\end{algorithm}

\begin{lemma}
\label{lemma: orth rank}
Let $\X=\U\mathbf{\Sigma}\V^T$ be the skinny SVD of $\X$, $\orthc(\X)=\U$, and $\orthr(\X)=\V^T$. Then for any set of coordinates $\Omega$ and any matrix $\X\in\R^{m\times n}$, we have
\begin{equation*}
\rank(\X_{\Omega:})=\rank([\orthc(\X)]_{\Omega:})\ \ \ \ \text{and}\ \ \ \ \rank(\X_{:\Omega})=\rank([\orthr(\X)]_{:\Omega}).
\end{equation*}
\end{lemma}
\begin{proof}
Let $\X=\U\mathbf{\Sigma}\V^T$ be the skinny SVD of matrix $\X$, where $\U=\orthc(\X)$ and $\V^T=\orthr(\X)$. On one hand,
\begin{equation*}
\X_{\Omega:}=\I_{\Omega:}\X=\I_{\Omega:}\U\mathbf{\Sigma}\V^T=[\orthc(\X)]_{\Omega:}\mathbf{\Sigma}\V^T.
\end{equation*}
So $\rank(\X_{\Omega:})\le\rank([\orthc(\X)]_{\Omega:})$. On the other hand, we have
\begin{equation*}
\X_{\Omega:}\V\mathbf{\Sigma}^{-1}=[\orthc(\X)]_{\Omega:}.
\end{equation*}
Thus $\rank([\orthc(\X)]_{\Omega:})\le\rank(\X_{\Omega:})$. So $\rank(\X_{\Omega:})=\rank([\orthc(\X)]_{\Omega:})$.

The second part of the argument can be proved similarly. Indeed, $\X_{:\Omega}=\U\mathbf{\Sigma}\V^T\I_{:\Omega}=\U\mathbf{\Sigma}[\orthr(\X)]_{:\Omega}$ and $\mathbf{\Sigma}^{-1}\U^T\X_{:\Omega}=[\orthr(\X)]_{:\Omega}$. So $\rank(\X_{:\Omega})=\rank([\orthr(\X)]_{:\Omega})$, as desired.
\end{proof}

\begin{proposition}
\label{proposition: subsampling with same rank}
Let $\L\in\R^{m\times n}$ be any rank-$r$ matrix with skinny SVD $\U\mathbf{\Sigma}\V^T$. Denote by $\L_{:\Omega}$ the submatrix formed by subsampling the columns of $\L$ with i.i.d. Ber($d/n$). If $d\ge 8\mu(\V)r\log(r/\delta)$, then with probability at least $1-\delta$, we have $\rank(\L_{:\Omega})=r$. Similarly, denote by $\L_{\Omega:}$ the submatrix formed by subsampling the rows of $\L$ with i.i.d. Ber($d/m$). If $d\ge 8\mu(\U)r\log(r/\delta)$, then with probability at least $1-\delta$, we have $\rank(\L_{\Omega:})=r$.
\end{proposition}

\begin{proof}
We only prove the first part of the argument. For the second part, applying the first part to matrix $\L^T$ gets the result. Denote by $\T$ the matrix $\V^T=\orthr(\L)$ with orthonormal rows, and by $\X=\sum_{i=1}^n\delta_i\T_{:i}e_i^T\in\R^{r\times n}$ the sampling of columns from $\T$ with $\delta_i\sim\mbox{Ber}(d/n)$. Let $\X_i=\delta_i\T_{:i}e_i^T$. Define positive semi-definite matrix
\begin{equation*}
\Y=\X\X^T=\sum_{i=1}^n\delta_i\T_{:i}\T_{:i}^T.
\end{equation*}
Obviously, $\sigma_r^2(\X)=\lambda_r(\Y)$. To invoke the matrix Chernoff bound, we estimate the parameters $L$ and $\mu_r$ in Lemma \ref{lemma: matrix chernoff bound}. Specifically, note that
\begin{equation*}
\mathbb{E}\Y=\sum_{i=1}^n\mathbb{E}\delta_i\T_{:i}\T_{:i}^T=\frac{d}{n}\sum_{i=1}^n\T_{:i}\T_{:i}^T=\frac{d}{n}\T\T^T.
\end{equation*}
Therefore, $\mu_r=\lambda_r(\mathbb{E}\Y)=d\sigma_r^2(\T)/n>0$. Furthermore, we also have
\begin{equation*}
\lambda_{\max}(\X_i)=\|\delta_i\T_{:i}\|_2^2\le\|\T\|_{2,\infty}^2\triangleq L.
\end{equation*}
By the matrix Chernoff bound where we set $\epsilon=1/2$,
\begin{equation*}
\begin{split}
\Pr\left[\sigma_r(\X)> 0\right]&=\Pr\left[\lambda_r(\Y)> 0\right]\\
& \ge \Pr\left[\lambda_r(\Y)>\frac{1}{2}\mu_r\right]\\
& = \Pr\left[\lambda_r(\Y)>\frac{d}{2n}\sigma_r^2(\T)\right]\\
&\ge 1-r\exp\left(-\frac{d\sigma_r^2(\T)}{8n\|\T\|_{2,\infty}^2}\right)\\
&\triangleq 1-\delta.
\end{split}
\end{equation*}
So if
\begin{equation*}
d\ge \frac{8n\|\T\|_{2,\infty}^2}{\sigma_r^2(\T)}\log\frac{r}{\delta}=8n\|\T\|_{2,\infty}^2\log\left(\frac{r}{\delta}\right),
\end{equation*}
then $\Pr\left[\sigma_{k+1}(\X)=0\right]\le\delta$, where the last equality holds since $\sigma_r(\T)=\sigma_r(\V^T)=1$. Note that
\begin{equation*}
\|\T\|_{2,\infty}^2\le \max_{i\in[n]}\|\V^T \e_i\|_2^2\le \frac{r}{n}\mu(\V).
\end{equation*}
So if $d\ge 8\mu(\V)r\log (r/\delta)$ then with probability at least $1-\delta$, $\rank(\T_{:\Omega})=r$. Also, by Lemma \ref{lemma: orth rank}, $\rank(\T_{:\Omega})=\rank([\orthr(\L)]_{:\Omega})=\rank(\L_{:\Omega})$. Therefore, $\rank(\L_{:\Omega})=r$ with a high probability, as desired.
\end{proof}

\comment{
\begin{lemma}[Noiseless Case]
\label{lemma: noiseless rank argument}
Let $\U^k\in\R^{m\times k}$ be a $k$-dimensional subspace of $\U^r$, $\L_{:t}\in\U^r$ but $\L_{:t}\not\in \U^k$. Suppose we get access to a set of coordinates $\Omega\subset [m]$ of size $d$ uniformly at random with replacement. If $d\ge 8\mu_0r\log((k+1)/\delta)$ then with probability at least $1-\delta$, $\rank\left([\U_{\Omega:}^k,\L_{\Omega t}]\right)=k+1$. Inversely, if $\L_{:t}\in \U^k$ and $d\ge 8\mu_0r\log((k+1)/\delta)$, then $\rank\left([\U_{\Omega:}^k,\M_{\Omega t}]\right)=k$ with probability $1$, and $\U^k\U_{\Omega:}^{k\dag}\L_{\Omega t}=\L_{: t}$ with probability at least $1-\delta$.
\end{lemma}
\begin{proof}
The former part of the argument is an immediate result from Proposition \ref{proposition: subsampling with same rank}, where we have used the fact that $\mu_0r\ge \mu([\U^k,\L_{: t}])(k+1)$.

The latter part of the argument is from the fact that $\L_{:t}\in\U^k$ implies $\L_{\Omega t}\in \U_{\Omega:}^k$. Furthermore, from Proposition \ref{proposition: subsampling with same rank}, we have $\rank(\U_{\Omega:}^k)=\rank(\U^k)=k$ with probability $1-\delta$. So $(\U_{\Omega:}^{kT}\U_{\Omega:}^k)^{-1}$ exists. Let $\L_{:t}=\U^kv$ since $\L_{:t}\in\U^k$. Then $\L_{\Omega t}=\U_{\Omega:}^kv$. Therefore, we have that $\U^k\U_{\Omega:}^{k\dag}\L_{\Omega t}=\U^k\U_{\Omega:}^{k\dag}\L_{\Omega t}=\U^k(\U_{\Omega:}^{kT}\U_{\Omega:}^k)^{-1}\U_{\Omega:}^{kT}\L_{\Omega t}=\U^k(\U_{\Omega:}^{kT}\U_{\Omega:}^k)^{-1}\U_{\Omega:}^{kT}\U_{\Omega :}^kv=\U^kv=\L_{:t}$, as desired.
\end{proof}
}

\comment{
Lemma \ref{lemma: noiseless rank argument} is for the noise-free setting. To extend to the case of sparse corruption, note that Lemma \ref{lemma: facts on non-degenerate distribution} implies that replacing $\U^k$ with $[\E^s,\U^k]$ will not influence the result of Lemma \ref{lemma: noiseless rank argument} due to the randomness of $\E^s$. Indeed, the idea comes from the fact that random vectors from non-degenerate distribution do not concentrate on any specific subspace. So expanding $\U^k$ by $\E^s$ does not affect the representation of any fixed vector $\M_{:t}$: If $\M_{:t}$ can be represented by $\U^k$, so can $[\E^s,\U^k]$; Inversely, if $\M_{:t}$ cannot be represented by $\U^k$, it cannot be a linear combination of $[\E^s,\U^k]$ as well. Formally, our results are as follows:}

\begin{lemma}
\label{lemma: outlier rank argument}
Let $\U^k\in\R^{m\times k}$ be a $k$-dimensional subspace of $\U^r$. Suppose we get access to a set of coordinates $\Omega\subset [m]$ of size $d$ uniformly at random without replacement. Let $s\le d-r-1$ and $d\ge c\mu_0r\log(k/\delta)$ for a universal constant $c$.
\begin{itemize}
\item
If $\M_{:t}\in \U^r$ but $\M_{:t}\not\in \U^k$ then with probability at least $1-\delta$, $\rank\left([\E_{\Omega:}^s,\U_{\Omega:}^k,\M_{\Omega t}]\right)=s+k+1$.
\item
If $\M_{:t}\in \U^k$, then $\rank\left([\E_{\Omega:}^s,\U_{\Omega:}^k,\M_{\Omega t}]\right)=s+k$ with probability $1$, the representation coefficients of $\M_{:t}$ corresponding to $\E^s$ in the dictionary $[\E^s,\U^k]$ is $\0$ with probability $1$, and $[\E^s,\U^k][\E_{\Omega:}^{s},\U_{\Omega:}^{k}]^{\dag}\M_{\Omega t}=\M_{: t}$ with probability at least $1-\delta$.
\item
If $\M_{:t}\not\in \U^r$, i.e., $\M_{:t}$ is an outlier drawn from a non-degenerate distribution, then $\rank\left([\E_{\Omega:}^s,\U_{\Omega:}^k,\M_{\Omega t}]\right)=s+k+1$ with probability $1-\delta$.
\end{itemize}
\end{lemma}
\begin{proof}
\comment{
For the first part of the lemma, by Fact 3 of Lemma \ref{lemma: facts on non-degenerate distribution}, the rank of the matrix $[\E_{\Omega:}^s,\U_{\Omega:}^k,\M_{\Omega t}]$ is determined by that of $[\U_{\Omega:}^k,\M_{\Omega t}]$, namely, they have relation $\rank([\E^s,\U^k,\M_{:t}]_{\Omega:})=s+\rank([\U^k,\M_{:t}]_{\Omega:})$.}

For the first part of the lemma, note that $\rank([\U^k,\M_{:t}])=k+1$. So according to Proposition \ref{proposition: subsampling with same rank}, with probability $1-\delta$ we have that $\rank([\U^k,\M_{:t}]_{\Omega :})=k+1$ since $d\ge c\mu_0r\log((k+1)/\delta)\ge 8\mu([\U^k,\M_{:t}])k\log((k+1)/\delta)$ (Because $\M_{:t}\in\U^r$). Recall Facts 3 and 4 of Lemma \ref{lemma: facts on non-degenerate distribution} which imply that $\rank([\E^s,\U^k,\M_{:t}]_{\Omega :})=s+k+1$ when $s\le d-r-1$. This is what we desire.

\comment{
For the first part of the lemma, by Facts 2 and 4 of Lemma \ref{lemma: facts on non-degenerate distribution}, the $\E_{\Omega :}^s$ cannot linearly represent the given vector $\M_{\Omega t}$ due to the randomness. Thus the problem of whether $\M_{\Omega t}$ can be represented by a given dictionary $[\E_{\Omega:}^s,\U_{\Omega:}^k]$ is totally determined by whether it can be represented by $\U_{\Omega:}^k$, i.e., whether $\rank([\U_{\Omega:}^k,\M_{\Omega t}])=k+1$. By Proposition \ref{proposition: subsampling with same rank}, we have that once $d\ge 8\mu([\U^k,\M_{:t}]) (k+1)\log((k+1)/\delta)$, $\rank([\U_{\Omega:}^k,\M_{\Omega t}])=\rank([\U^k,\M_{:t}])=k+1$. Finally, the conclusion follows from the property that $\mu_0r\ge \mu([\U^k,\M_{: t}])(k+1)$.}

\comment{
The first part of the lemma is by Lemma \ref{lemma: noiseless rank argument} and the fact that $\rank([\E^s,\U^k,\M_{:t}]_{\Omega:})=s+\rank([\U^k,\M_{:t}]_{\Omega:})$ (Fact 3 of Lemma~\ref{lemma: facts on non-degenerate distribution}).}

For the middle part, the statement $\rank\left([\E_{\Omega:}^s,\U_{\Omega:}^k,\M_{\Omega t}]\right)=s+k$ comes from the assumption that $\M_{:t}\in\U^k$, which implies that $\M_{\Omega t}\in\U_{\Omega :}^k$ with probability $1$, and that $\rank\left([\E_{\Omega:}^s,\U_{\Omega:}^k]\right)=s+k$ when $s\le d-r-1$ (Facts 3 and 4 of Lemma \ref{lemma: facts on non-degenerate distribution}). Now suppose that the representation coefficients of $\M_{:t}$ corresponding to $\E^s$ in the dictionary $[\E^s,\U^k]$ is NOT $\0$ and $\M_{:t}\in \U^k$. Then $\M_{:t}-\U^k\c\in\mathbf{span}(\E^s)$, where $\c$ is the representation coefficients of $\M_{:t}$ corresponding to $\U^k$ in the dictionary $[\E^s,\U^k]$. Also, note that $\M_{:t}-\U^k\c\in\U^k$. So $\rank[\E^s,\M_{:t}-\U^k\c]=s$, which is contradictory with Fact 2 of Lemma \ref{lemma: facts on non-degenerate distribution}. So the coefficient w.r.t. $\E^s$ in the dictionary $[\E^s,\U^k]$ is $\0$, and we have that $[\E^s,\U^k][\E_{\Omega:}^{s},\U_{\Omega:}^{k}]^{\dag}\M_{\Omega t}=\U^k\U_{\Omega:}^{k\dag}\M_{\Omega t}=\U^k(\U_{\Omega:}^{kT}\U_{\Omega:}^k)^{-1}\U_{\Omega:}^{kT}\M_{\Omega t}=\U^k(\U_{\Omega:}^{kT}\U_{\Omega:}^k)^{-1}\U_{\Omega:}^{kT}\U_{\Omega :}^k\v=\U^k\v=\M_{: t}$, where $\v$ is the representation coefficient of $\M_{:t}$ w.r.t. $\U^k$. (The $(\U_{\Omega:}^{kT}\U_{\Omega:}^k)^{-1}$ exists because $\rank(\U_{\Omega:}^k)=k$ by Proposition \ref{proposition: subsampling with same rank})

As for the last part of the lemma, note that by Facts 2 and 4 of Lemma \ref{lemma: facts on non-degenerate distribution}, $\rank([\E^s,\M_{:t}]_{\Omega:})=s+1$. Then by Fact 3 of Lemma \ref{lemma: facts on non-degenerate distribution} and the fact that $\U_{\Omega:}^k$ has rank $k$ (Proposition \ref{proposition: subsampling with same rank}), we have $\rank\left([\E_{\Omega:}^s,\U_{\Omega:}^k,\M_{\Omega t}]\right)=s+k+1$ when $s\le d-r-1$, as desired.
\end{proof}

Now we are ready to prove Theorem \ref{theorem: upper bound in noiseless case}.

\noindent{\textbf{Theorem ~\ref{theorem: upper bound in noiseless case} (restated).}}
\emph{Let $r$ be the rank of the underlying matrix $\L$ with $\mu_0$-incoherent column space. Suppose that the outliers $\E^{s_0}\in\R^{m\times s_0}$ are drawn from any non-degenerate distribution, and that the underlying subspace $\U^r$ is identifiable. Then Algorithm \ref{algorithm: never ending learning sparse noise} exactly recovers the underlying matrix $\L$, the column space $\U^r$, and the outlier $\E^{s_0}$ with probability at least $1-\delta$, provided that $d\ge c\mu_0r\log\left(\frac{r}{\delta}\right)$ and $s_0\le d-r-1$. The total sample complexity is thus $c\mu_0rn\log\left(\frac{r}{\delta}\right)$, where $c$ is a universal constant.}

\begin{proof}
The proof of Theorem \ref{theorem: upper bound in noiseless case} is an immediate result of Lemma \ref{lemma: outlier rank argument} by using the union bound on the samplings of $\Omega$. Although Lemma \ref{lemma: outlier rank argument} states that, for a specific column $\M_{:t}$, the algorithm succeeds with probability at least $1-\delta$, the probability of success that uniformly holds for all columns is $1-(r+s_0)\delta$ rather than $1-n\delta$. This observation is from the proof of Lemma \ref{lemma: outlier rank argument}: $[\E^s,\U^k][\E_{\Omega:}^{s},\U_{\Omega:}^{k}]^{\dag}\M_{\Omega t}=\M_{: t}$ holds so long as $(\U_{\Omega:}^{kT}\U_{\Omega:}^k)^{-1}$ exists. Since in Algorithm \ref{algorithm: never ending learning sparse noise} we resample $\Omega$ if and only if we add new vectors into the basis matrix, which happens at most $r+s_0$ times, the conclusion follows from the union bound of the $r+s_0$ events. Thus, to achieve a global probability of $1-\delta$, the sample complexity for each upcoming column is $\Theta(\mu_0r\log(r+s_0/\delta))$. Since we also require that $s_0\le d-r-1$, the algorithm succeeds with probability $1-\delta$ once $d\ge \Theta(\mu_0r\log(d/\delta))$. Solving for $d$, we obtain that $d\gtrsim \mu_0r\log(\mu_0^2r^2/\delta^2)\asymp \mu_0r\log(r/\delta)$\footnote{We assume here that $\mu_0\le \mbox{poly}(r/\delta)$.}. The total sample complexity for Algorithm \ref{algorithm: never ending learning sparse noise} is thus $\Theta(\mu_0rn\log(r/\delta))$.

For the exact identifiability of the outliers, we have the following guarantee:

\begin{lemma}[Outlier Removal]
\label{lemma: outlier removal}
Let the underlying subspace $\U^r$ be identifiable, i.e., for each $i\in [r]$, there are at least two columns $\M_{:a_i}$ and $\M_{:b_i}$ of $\M$ such that $[\orthc(\U^r)]_{:i}^T\M_{:a_i}\not=0$ and $[\orthc(\U^r)]_{:i}^T\M_{:b_i}\not=0$. Then the entries of $\C$ in Algorithm \ref{algorithm: never ending learning sparse noise} corresponding to $\U^r$ cannot be $0$'s.
\end{lemma}

\begin{proof}
Without loss of generality, let $\U^r$ be orthonormal. Suppose that the lemma does not hold true. Then there must exist one column $\U_{:i}^r$ of $\U^r$, say e.g., $\e_i$, such that $\e_i^T\M_{:t}=0$ for all $t$ except when the index $t$ corresponds exactly to the $\U_{:i}^r$. This is contradictory with the condition that the subspace $\U^r$ is identifiable. The proof is completed.
\end{proof}

Thus the proof of Theorem \ref{theorem: upper bound in noiseless case} is completed.
\end{proof}
}

\comment{
\section{Proof of Lower Bound for Exact Recovery}

\noindent{\textbf{Theorem ~\ref{theorem: lower bound in noiseless case} (restated).}}
\emph{Let $0<\delta<1/2$, and $\Omega\sim \mathrm{Uniform}(d)$ be the index of the row sampling $\subseteq [m]$. Suppose that $\U^r$ is $\mu_0$-incoherent. If the total sampling number
$dn< c\mu_0rn\log\left(r/\delta\right)$ for a constant $c$, then with probability at least $1-\delta$, there is an example of $\M$ such that under the sampling model of Section 2.1 in the main body (i.e., when a column arrives the choices are either (a) randomly sample or (b) view the entire column), there exist infinitely many matrices $\L'$ of rank $r$ obeying $\mu_0$-incoherent condition on column space such that $\L'_{\Omega:}=\L_{\Omega:}$.}

\begin{proof}
We prove the theorem by assuming that the underlying column space is known. Since we require additional samples to estimate the subspace, the proof under this assumption gives a lower bound. Let $\ell=\left\lfloor\frac{m}{\mu_0r}\right\rfloor$. Construct the underlying matrix $\L$ by
\begin{equation*}
\L=\sum_{k=1}^rb_k\u_k\u_k^T,
\end{equation*}
where the known $\u_k$ (Because the column space is known) is defined as
\begin{equation*}
\u_k=\sqrt{\frac{1}{\ell}}\sum_{i\in B_k}\e_i,\ \ \ \ B_k=\{(k-1)\ell+1,(k-1)\ell+2,...,k\ell\}.
\end{equation*}
So the matrix $\L$ is a block diagonal matrix formulated as Figure \ref{figure: construction of L}.
\begin{figure}[ht]
\centering
\includegraphics[width=0.7\textwidth]{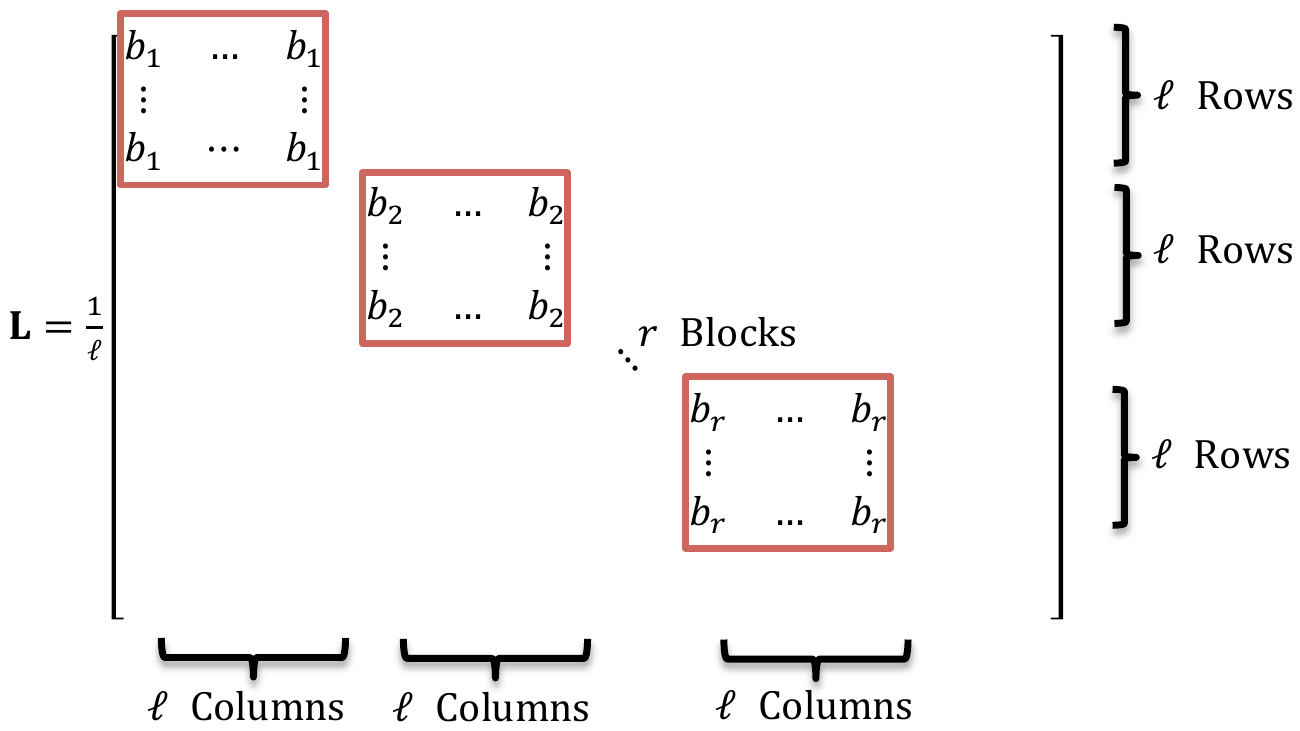}
\caption{Construction of underlying matrix $\L$.}
\label{figure: construction of L}
\end{figure}
Further, construct the noisy matrix $\M$ by
$
\M=[\L,\E].
$
The matrix $\E\in\R^{m\times s_0}$ corresponds to the outliers, and the matrix $\L$ corresponds to the underlying matrix.

Notice that the information of $b_k$'s is only implied in the corresponding block of $\L$.
So overall, the lower bound is given by solving from the inequality
\begin{equation*}
\Pr\{\mbox{For all blocks, there must be at least one row being sampled}\}\ge 1-\delta.
\end{equation*}
We highlight that the $b_k$'s can be chosen arbitrarily in that they do not change the coherence of the column space of $\L$. Also, it is easy to check that the column space of $\L$ is $\mu_0$-incoherent. By construction, the underlying matrix $\L$ is block-diagonal with $r$ blocks, each of which is of size $\ell\times \ell$. According to our sampling scheme, we always sample the same positions of the arriving column after the column space is known to us. This corresponds to sample the row of the matrix in hindsight. To recover $\L$, we argue that each block should have at least one row fully observed; Otherwise, there is no information to recover $b_k$'s. Let $A$ be the event that for a fixed block, none of its rows is observed. The probability $\pi_0$ of this event $A$ is therefore $\pi_0=(1-p)^\ell$, where $p$ is the Bernoulli sampling parameter. Thus by independence, the probability of the event that there is at least one row being sampled holds true \emph{for all diagonal blocks} is $(1-\pi_0)^{r}$, which is $\ge 1-\delta$ as we have argued. So
\begin{equation*}
-r\pi_0\ge r\log(1-\pi_0)\ge \log(1-\delta),
\end{equation*}
where the first inequality is due to the fact that $-x\ge \log(1-x)$ for any $x<1$. Since we have assumed $\delta<1/2$, which implies that $\log(1-\delta)\ge -2\delta$, thus $\pi_0\le 2\delta/r$. Note that $\pi_0=(1-p)^\ell$, and so
\begin{equation*}
-\log(1-p)\ge \frac{1}{\ell}\log\left(\frac{r}{2\delta}\right)\ge\frac{\mu_0 r}{m}\log\left(\frac{r}{2\delta}\right).
\end{equation*}
This is equivalent to
\begin{equation*}
mp\ge m\left(1-\exp\left(-\frac{\mu_0 r}{m}\log \frac{r}{2\delta}\right)\right).
\end{equation*}
Note that $1-e^{-x}\ge x-x^2/2$ whenever $x\ge 0$, we have
\begin{equation*}
mp\ge (1-\epsilon/2)\mu_0r\log\left(\frac{r}{2\delta}\right),
\end{equation*}
where $\epsilon=\mu_0r\log(r/2\delta)<1$. Finally, by the equivalence between the uniform and Bernoulli sampling models (i.e., $d\approx mp$, Lemma \ref{lemma: number of sampling by Bernoulli}), the proof is completed.
\end{proof}
}

\comment{
\section{Mixture of Subspaces}
In this section, we assume that the underlying subspace is a mixture of $h$ independent subspaces, each of which is of dimension at most $\tau\ll r$. Such an assumption naturally models settings in which there are really $h$ different categories of data while they share a certain commonality across categories. Indeed, many real datasets match the assumption, e.g., face image, 3D motion trajectory, handwritten digits, etc.

\begin{algorithm}[ht]
\caption{Noise-Tolerant Life-Long Matrix Completion under Random Noise for Mixture of Subspaces}
\begin{algorithmic}
\label{algorithm: never ending learning sparse noise mixture of subspaces}
\STATE {\bfseries Input:} Columns of matrices arriving over time.
\STATE {\bfseries Initialize:} Let the basis matrix $\widehat{\B}^0=\emptyset$, the counter $\C=\emptyset$. Randomly draw entries $\Omega\subset [m]$ of size $d$ uniformly without replacement.
\STATE {\bfseries 1:} \textbf{For} each column $t$ of $\M$, \textbf{do}
\STATE {\bfseries 2:}\ \ \ \ \ \ \ \ (a) If there does not exist a $\tau$-sparse linear combination of columns of $\widehat{\B}_{\Omega :}^k$ that represents $\M_{\Omega t}$ exactly
\STATE {\bfseries 3:}\ \ \ \ \ \ \ \ \ \ \ \ \ i. Fully measure $\M_{:t}$ and add it to the basis matrix $\widehat{\B}^k$.
\STATE {\bfseries 4:}\ \ \ \ \ \ \ \ \ \ \ \ \ ii. $\C:=[\C,0]$.
\STATE {\bfseries 5:}\ \ \ \ \ \ \ \ \ \ \ \ \ iii. Randomly draw entries $\Omega\subset [m]$ of size $d$ uniformly without replacement.
\STATE {\bfseries 6:}\ \ \ \ \ \ \ \ \ \ \ \ \ iv. $k:=k+1$.
\STATE {\bfseries 7:}\ \ \ \ \ \ \ \ (b) Otherwise
\STATE {\bfseries 8:}\ \ \ \ \ \ \ \ \ \ \ \ \ i. $\C:=\C+\1_{supp(\widehat{\B}_{\Omega:}^{k\dag}\M_{\Omega t})}^T$.\ \ \ \ //Record supports of representation coefficient
\STATE {\bfseries 9:}\ \ \ \ \ \ \ \ \ \ \ \ \ ii. $\widehat{\M}_{:t}:=\widehat{\B}^k\widehat{\B}_{\Omega:}^{k\dag}\M_{\Omega t}$.
\STATE {\bfseries 10:}\ \ \ \ \ \ \ $t:= t+1$.
\STATE {\bfseries 11:} \textbf{End For}
\STATE {\bfseries Outlier Removal:} Remove columns corresponding to entry 0 in vector $\C$ from $\widehat{\B}^{s_0+r}=[\E^{s_0},\U^r]$.
\STATE {\bfseries Output:} Estimated range space, identified outlier vectors, and recovered underlying matrix $\widehat{\M}$ with column $\widehat{\M}_{:t}$.
\end{algorithmic}
\end{algorithm}

\begin{lemma}[Mixture of Subspaces]
\label{lemma: outlier rank argument mixture of subspaces}
Let $[\E^s,\U^k]$ be the current dictionary matrix consisting of a random noise matrix $\E^s\in\R^{m\times s}$ and a clean basis matrix $\U^k\in\R^{m\times k}$. Suppose we get access to a set of coordinates $\Omega\subset [m]$ of size $d$ uniformly at random without replacement. Let $s\le d-\tau-1$ and $d\ge 8\mu_\tau \tau\log(\tau/\delta)$. Denote by $\U^\tau\in\R^{m\times \tau}$ a submatrix of $\U^k$ with $\tau$ columns.
\begin{itemize}
\item
If $\M_{:t}\in \U^r$ but it cannot be represented by a linear combination of $\tau$ vectors in the current dictionary, then with probability at least $1-\delta$, $\M_{\Omega t}$ does not belong to any fixed $\tau$-combination of the truncated dictionary as well.
\item
If $\M_{:t}$ can be represented by a linear combination of $\tau$ vectors in the current basis, then $\M_{\Omega t}$ can be represented as a linear combination of the same $\tau$ truncated vectors in the dictionary with probability $1$, the representation coefficients of $\M_{:t}$ corresponding to $\E^s$ in the dictionary is $\0$ with probability $1$, and $[\E^s,\U^k][\E_{\Omega:}^{s},\U_{\Omega:}^{k}]^{\dag}\M_{\Omega t}=\M_{: t}$ with probability at least $1-\delta$.
\item
If $\M_{:t}$ is an outlier drawn from a non-degenerate distribution, then $\M_{\Omega t}$ cannot be represented by the dictionary with probability $1$.
\end{itemize}
\end{lemma}

\begin{proof}
The proof is similar as that of Lemma \ref{lemma: outlier rank argument}. For completeness, we give a brief proof here. For the first part of the lemma, by Facts 2 and 4 of Lemma \ref{lemma: facts on non-degenerate distribution}, the $\E_{\Omega :}^s$ cannot have a non-zero representation coefficient of $\M_{\Omega t}$ in any possible $\tau$-combination of the current dictionary when $s\le d-\tau-1$, due to the randomness. Thus the problem of whether $\M_{\Omega t}$ can be $\tau$ represented by the current dictionary $[\E_{\Omega:}^s,\U_{\Omega:}^k]$ is totally determined by whether it can be $\tau$ represented by $\U_{\Omega:}^k$. Now suppose that $\M_{\Omega t}$ can be written as a linear $\tau$-combination of the current basis $\U_{\Omega :}^\tau$. Then according to Proposition \ref{proposition: subsampling with same rank}, since $d\ge 8\mu_\tau\tau\log(\tau/\delta)$, we have that $\rank([\U^\tau,\M_{:t}])=\tau$, which is contradictory with the assumption of Event 1.

The first argument in Event 2 is obvious. Now suppose that the representation coefficients of $\M_{:t}$ corresponding to $\E^s$ in the dictionary $[\E^s,\U^k]$ is NOT $\0$ and $\M_{:t}\in \U^k$. Then $\M_{:t}-\U^k\c\in\mathbf{span}(\E^s)$, where $\c$ is the representation coefficients of $\M_{:t}$ corresponding to $\U^k$ in the dictionary $[\E^s,\U^k]$. Also, note that $\M_{:t}-\U^k\c\in\U^k$. So $\rank[\E^s,\M_{:t}-\U^k\c]=s$, which is contradictory with Fact 2 of Lemma \ref{lemma: facts on non-degenerate distribution}. So the coefficient w.r.t. $\E^s$ in the dictionary $[\E^s,\U^k]$ is $\0$. Since by assumption $\M_{:t}$ can be represented by $\tau$ combination of columns in $\U^k$, termed $\U^\tau$, we have that $[\E^s,\U^k][\E_{\Omega:}^{s},\U_{\Omega:}^{k}]^{\dag}\M_{\Omega t}=\U^\tau\U_{\Omega:}^{\tau\dag}\M_{\Omega t}=\U^\tau(\U_{\Omega:}^{\tau T}\U_{\Omega:}^\tau)^{-1}\U_{\Omega:}^{\tau T}\M_{\Omega t}=\U^\tau(\U_{\Omega:}^{\tau T}\U_{\Omega:}^\tau)^{-1}\U_{\Omega:}^{\tau T}\U_{\Omega :}^\tau \v=\U^\tau \v=\M_{: t}$, where $\v$ is the representation coefficient of $\M_{:t}$ w.r.t. $\U^\tau$. (The $(\U_{\Omega:}^{\tau T}\U_{\Omega:}^\tau)^{-1}$ exists because $\rank(\U_{\Omega:}^\tau)=\tau$ by Proposition \ref{proposition: subsampling with same rank})

Event 3 is an immediate result of Lemma \ref{lemma: facts on non-degenerate distribution}.
\end{proof}

\noindent{\textbf{Theorem ~\ref{theorem: upper bound in noiseless case for mixture of subspaces} (restated).}}
\emph{Let $r$ be the rank of the underlying matrix $\L$. Suppose that the columns of $\L$ lie on a mixture of independent subspaces, each of which is of dimension at most $\tau$. Denote by $\mu_\tau$ the maximal incoherence over all $\tau$-combinations of $\L$. Let the noise model be that of Theorem \ref{theorem: upper bound in noiseless case}. Then Algorithm \ref{algorithm: never ending learning sparse noise mixture of subspaces} exactly recovers the underlying matrix $\L$, the column space $\U^r$, and the outlier $\E^{s_0}$ with probability at least $1-\delta$, provided that $d\ge c\mu_\tau \tau^2\log\left(\frac{r}{\delta}\right)$ for some global constant $c$ and $s_0\le d-\tau-1$. The total sample complexity is thus $c\mu_\tau \tau^2 n\log\left(\frac{r}{\delta}\right)$.}

\begin{proof}
Theorem \ref{theorem: upper bound in noiseless case for mixture of subspaces} is a result of union bound of Lemma \ref{lemma: outlier rank argument mixture of subspaces}. For the event of type 1, the union bound is over $\binom{r}{\tau}=\O(r^\tau)$ events. For the event of type 2, since we resample $\Omega$ at most $r+s_0$ times by algorithm, the union bound is over $r+s_0$ samplings. The event of type 3 is with probability $1$. So overall, replacing $\delta$ with $\min\{\delta/r^\tau,\delta/(r+s_0)\}$ in Lemma \ref{lemma: outlier rank argument mixture of subspaces}, the sample complexity we need is at least $\O(\mu_\tau\tau\log(\max\{r^\tau,r+s_0\}/\delta))$. Note that $s_0\le d-\tau-1$. So the sample complexity for each column is at least $\O(\mu_\tau\tau^2\log(r/\delta))$ and the total one is $\O(\mu_\tau\tau^2n\log(r/\delta))$, as desired. The success of outlier removal step is guaranteed by Lemma \ref{lemma: outlier removal}.
\end{proof}

\comment{
The proof is same as that of Lemma \ref{lemma: outlier rank argument} by replacing the $\E^s$ in Lemma \ref{lemma: outlier rank argument} with $[\U_{-i}^b,\E^s]$, because $[\U_{-i}^b,\E^s]$ cannot represent any vector in $\U_i^\tau$ (Since the $d$ different subspaces are disjoint).}

}

\section{Facts on Subspace Spanned by Non-Degenerate Random Vectors}
\begin{lemma}
\label{lemma: facts on non-degenerate distribution}
Let $\E^s\in\R^{m\times s}$ be matrix consisting of corrupted vectors drawn from any non-degenerate distribution. Let $\U^k\in\R^{m\times k}$ be any fixed matrix with rank $k$. Then with probability $1$, we have
\begin{itemize}[leftmargin=*]
\item
$\rank(\E^s)=s$ for any $s\le m$;
\item
$\rank([\E^s,\x])=s+1$ holds for $\x\in \U^k\subset \R^m$ uniformly and $s\le m-k$, where $\x$ can even depend on $\E^s$;
\item
$\rank([\E^s,\U^k])=s+k$, provided that $s+k\le m$;
\item
The marginal of non-degenerate distribution is non-degenerate.
\end{itemize}
\end{lemma}
\begin{proof}
For simplicity, we only show the proof of Fact 1. The other facts can be proved similarly. Let $\E^s=[\E^{s-1},\e]$.  Since $\e$ is drawn from a non-degenerate distribution, the conditional probability satisfies $\Pr[\rank(\E^{s-1},\e)=s\mid \E^{s-1}]=1$ by the definition of non-degenerate distribution. So $\Pr[\rank(\E^{s-1},\e)=s]=\mathbb{E}_{\E^{s-1}}\Pr[\rank(\E^{s-1},\e)=s\mid \E^{s-1}]=1$.
\end{proof}

\section{Equivalence between Bernoulli and Uniform Models}
\begin{lemma}
\label{lemma: number of sampling by Bernoulli}
Let $n$ be the number of Bernoulli trials and suppose that $\Omega\sim\mbox{Ber}(d/n)$. Then with probability at least $1-\delta$, $|\Omega|=\Theta(d)$, provided that $d\ge 4\log (1/\delta)$.
\end{lemma}
\begin{proof}
Take a perturbation $\epsilon$ such that $d/n=d_0/n+\epsilon$. By the scalar Chernoff bound which states that
\begin{equation*}
\Pr(|\Omega|\le d_0)\le e^{-\epsilon^2n^2/2d_0},
\end{equation*}
if taking $d_0=d/2$, $\epsilon=d/2n$ and $d\ge 4\log (1/\delta)$, we have
\begin{equation}
\label{equ: up}
\Pr(|\Omega|\le d/2)\le e^{-d/4}\le \delta.
\end{equation}

In the other direction, by the scalar Chernoff bound again which states that
\begin{equation*}
\Pr(|\Omega|\ge d_0)\le e^{-\epsilon^2n^2/3d},
\end{equation*}
if taking $d_0=2d$, $\epsilon=-d/n$ and $d\ge 4\log (1/\delta)$, we obtain
\begin{equation}
\label{equ: lower}
\Pr(|\Omega|\ge 2d)\le e^{-d/3}\le \delta.
\end{equation}

Finally, according to \eqref{equ: up} and \eqref{equ: lower}, we conclude that $d/2<|\Omega|<2d$ with probability at least $1-\delta$.
\end{proof}

\section{A Collection of Concentration Results}
\begin{lemma}[Theorem 6.~\cite{krishnamurthy2014power}]
\label{lemma: concentration of measure}
Denote by $\widetilde{\U}^k$ a $k$-dimensional subspace in $\R^m$. Let the sampling number $d\ge \max\{\frac{8}{3}k\mu(\widetilde{\U}^k)\log(\frac{2k}{\delta}),4\mu(\P_{\widetilde{\U}^{k\perp}} y)\log(\frac{1}{\delta})\}$. Denote by $\Omega$ an index set of size $d$ sampled uniformly at random with replacement from $[m]$. Then with probability at least $1-4\delta$, for any $y\in\R^m$, we have
\begin{equation*}
\frac{d(1-\alpha)-k\mu(\widetilde{\U}^k)\frac{\beta}{1-\zeta}}{m}\left\|y-\P_{\widetilde{\U}^k}y\right\|_2^2\le\left\|y_{\Omega}-\P_{\widetilde{\U}_{\Omega:}^k}y_{\Omega}\right\|_2^2\le(1+\alpha)\frac{d}{m}\left\|y-\P_{\widetilde{\U}^k}y\right\|_2^2,
\end{equation*}
where $\alpha=\sqrt{2\frac{\mu(\P_{\widetilde{\U}^{k\perp}} y)}{d}\log(1/\delta)}+\frac{2\mu(\P_{\widetilde{\U}^{k\perp}} y)}{3d}\log(1/\delta)$, $\beta=(1+2\log (1/\delta))^2$, and $\zeta=\sqrt{\frac{8k\mu(\widetilde{\U}^k)}{3d}\log(2r/\delta)}$.
\end{lemma}

\begin{lemma}[Matrix Chernoff Bound.~\cite{gittens2011tail}]
\label{lemma: matrix chernoff bound}
Consider a finite sequence $\{\X_k\}\in\mathbb{R}^{n\times n}$ of independent, random, Hermitian matrices. Assume that
\begin{equation*}
0\le\lambda_{\min}(\X_k)\le\lambda_{\max}(\X_k)\le L.
\end{equation*}
Define $\Y=\sum_k \X_k$, and $\mu_r$ as the $r$-th largest eigenvalue of the expectation $\mathbb{E}\Y$, i.e., $\mu_r=\lambda_r(\mathbb{E}\Y)$. Then
\begin{equation*}
\Pr\left\{\lambda_r(\Y)>(1-\epsilon)\mu_r\right\}\ge 1-r\left[\frac{e^{-\epsilon}}{(1-\epsilon)^{1-\epsilon}}\right]^{\frac{\mu_r}{L}}\ge 1-r e^{-\frac{\mu_r\epsilon^2}{2L}}\ \ \mbox{for}\  \epsilon\in[0,1).
\end{equation*}
\end{lemma}

\comment{
\section{Paper Organization}
The paper is organized as follows:
\begin{itemize}
\item
Introduction
\begin{itemize}
\item
Matrix Completion, Life-Long Learning, Noise-Tolerant, emphasize challenge in our setting
\item
Technical Advantages
\begin{itemize}
\item
Incoherence only on column space,
Online/life-long:  inherit Aarti's paper
\item
Our Contributions: Optimal sample complexity
\item
Our Contributions: Noise-tolerant, control error propagation while online
\end{itemize}
\end{itemize}
\item
Related Work
\begin{itemize}
\item
Matrix completion (how to improve sample complexity)
\item
Life-Long Learning
\item
Matched Subspace Detection
\item
Column Subset Selection does not work
\end{itemize}
\item
Main Results
\begin{itemize}
\item
Deterministic Noise
\begin{itemize}
\item
Algorithm
\item
Upper Bound
\end{itemize}
\item
Random Noise
\begin{itemize}
\item
Algorithm
\item
Upper Bound
\item
Lower Bound
\end{itemize}
\end{itemize}
\item
Experiments
\item
Conclusion and Future Work
\item
Appendices
\end{itemize}
}

\end{document}